\documentclass[twoside,11pt]{article}
\usepackage{jair, theapa, rawfonts}

\usepackage[hyphens]{url}
\usepackage{import}
\usepackage{stmaryrd}
\usepackage{amsthm}
\usepackage{amsmath}
\usepackage{amssymb}

\usepackage{algorithmic}
\usepackage{url}
\usepackage{comment}
\usepackage{graphicx}
\usepackage{multirow}
\usepackage{float}
\usepackage{booktabs}
\usepackage{rotating}
\usepackage{subcaption}
\usepackage{mathtools}
\usepackage{adjustbox}
\usepackage{tikz}
\usetikzlibrary{matrix}
\usetikzlibrary{calc}
\usepackage{pgfplots}

\usepackage[textsize=tiny,
	disable,
	]{todonotes}
\usepackage{mathtools}
\usepackage{comment}

\usepackage{xcolor}

\usepackage{lscape}

\newcommand{\frm}[2][]{\todo[color=red!60,linecolor={red!100},#1,size=\tiny]{Felipe: #2}}

\newtheorem{example}{\textbf{Example}}
\newcommand{\citet}[1]{\citeauthor{#1}~\citeyear{#1}}
\theoremstyle{plain}
\newtheorem{proposition}{Proposition}

\newtheorem{theorem}{Theorem}
\theoremstyle{definition}
\newtheoremstyle{bfnote}
{}                %
{}                %
{\slshape}        %
{}                %
{\bf}       %
{.}               %
{ }               %
{\thmname{#1}\thmnumber{ #2}\thmnote{ (#3)}}       %
\theoremstyle{bfnote}
\newtheorem{definition}{Definition}

\newcommand{\sasplus}{SAS\textsuperscript{+}}

\newcommand{\dhc}{\ensuremath{\Gamma^{\text{LP}}}}

\newcommand{\rg}{RG}

\newcommand{\pom}{POM}

\newcommand{\seq}{\text{SEQ}}
\newcommand{\pho}{\text{PhO}}
\newcommand{\lmc}{\text{LMC}}

\newcommand{\flow}{\text{FLOW}}

\newcommand{\dr}{\text{DEL}}

\newcommand{\hseq}{\ensuremath{h^{\seq}}}
\newcommand{\hlmc}{\ensuremath{h^{\lmc}}}

\newcommand{\hpho}{\ensuremath{h^{\pho}}}
\newcommand{\hflow}{\ensuremath{h^{\flow}}}

\newcommand{\hdr}{\ensuremath{h^{\dr}}}
\newcommand{\hoptimal}{\ensuremath{\h^{*}}}

\newcommand{\hip}{\h^{\textup{IP}}}

\newcommand{\hooptg}{\ensuremath{h^{*}_{\observations,\goalstate}}}

\newcommand{\holmc}{\ensuremath{h^{\lmc}_{\observations}}}

\newcommand{\holmcsg}{\ensuremath{h^{\lmc_\observations}_{\observations,\goalstate}}}

\providecommand{\tuple}[1]{\ensuremath\langle#1\rangle}

\providecommand{\exec}[1]{\ensuremath{\llbracket{#1}\rrbracket}}

\providecommand\Z{\ensuremath{\mathbb{Z}}}

\providecommand\Real{\ensuremath{\mathbb{R}}}

\newcommand{\planningtask}{\ensuremath{\Pi}}
\newcommand{\variables}{\ensuremath{\mathcal{V}}}
\newcommand{\operators}{\ensuremath{\mathcal{O}}}
\newcommand{\states}{\ensuremath{S}}
\newcommand{\goalstates}{\ensuremath{\states_{*}}}
\newcommand{\facts}{\ensuremath{\mathcal{A}}}
\newcommand{\vertices}{\mathcal{S}}
\newcommand{\edges}{\mathcal{T}}

\providecommand{\cost}{\mathit{cost}}

\providecommand{\pre}{\mathit{pre}}
\providecommand{\post}{\mathit{post}}
\providecommand{\vars}{\mathit{vars}}

\providecommand{\occur}{\mathit{occur}}
\providecommand{\dom}[1]{\ensuremath{\emph{dom}({#1})}}

\newcommand{\transgraph}{\ensuremath{TS}}

\newcommand{\grplanningtask}{\ensuremath{\planningtask_\textup{P}}}
\newcommand{\grtask}{\ensuremath{\planningtask_{\goalconditions}^{\observations}}}
\newcommand{\grsolution}{\ensuremath{\Gamma^{*}}}
\newcommand{\observations}{\ensuremath{\Omega}}
\newcommand{\unreliability}{\ensuremath{\epsilon}}%

\providecommand{\h}{\ensuremath{h}}
\providecommand{\ho}{\ensuremath{h_{\observations}}}

\providecommand{\variable}{\ensuremath{V}} %

\providecommand{\obs}[1]{\ensuremath{\vec{#1}}}

\providecommand{\plan}{\ensuremath{\pi}}
\providecommand{\initialstate}{\ensuremath{s_{0}}}
\providecommand{\goalstate}{\ensuremath{s^{*}}}
\providecommand{\rgoal}{\ensuremath{\goalstate_\mathcal{R}}}
\providecommand{\goalconditions}{\ensuremath{\Gamma}}

\newcommand{\optgrplanningtask}{}

\newcommand{\hPi}[1]{\ensuremath{\h_{\optgrplanningtask #1}}}
\newcommand{\hstarPi}[1]{\ensuremath{\hoptimal_{\optgrplanningtask #1}}}
\newcommand{\hoPi}[2][\observations]{\ensuremath{\h_{\optgrplanningtask #2, #1}}}
\newcommand{\hostarPi}[2][\observations]{\ensuremath{\hoptimal_{\optgrplanningtask #2, #1}}}
\newcommand{\hoipPi}[2][\observations]{\ensuremath{\hip_{\optgrplanningtask #2, #1}}}

\providecommand\Y[1]{\ensuremath{\mathsf{Y}_{#1}}}
\providecommand\Yobs[1]{\ensuremath{\obs{\mathsf{Y}}_{#1}}}

\providecommand\varU[1]{\ensuremath{\mathsf{U}_{#1}}}
\providecommand\varR[1]{\ensuremath{\mathsf{R}_{#1}}}
\providecommand\varA[1]{\ensuremath{\mathsf{A}_{#1}}}
\providecommand\varT[1]{\ensuremath{\mathsf{T}_{#1}}}
\newcommand\lpvariables{\ensuremath{\mathcal{Y}}}
\newcommand\constraints{\ensuremath{C}}
\newcommand\constraintso{\ensuremath{C_{\observations}}}
\newcommand\landmark{\ensuremath{L}}

\jairheading{}{2024}{}{}{}
\ShortHeadings{Goal Recognition via Linear Programming}
{Meneguzzi, Santos, Pereira, \& Pereira}
\firstpageno{1}

\begin{document}

\title{Goal Recognition via Linear Programming}

\author{\name Felipe Meneguzzi \email felipe.meneguzzi@abdn.ac.uk \\
        \addr University of Aberdeen, Scotland, UK\\
        \addr Pontifical Catholic University of Rio Grande do Sul, Brazil
       \AND
       \name Ramon Fraga Pereira \email ramonfraga.pereira@manchester.ac.uk \\
       \addr University of Manchester, England, UK
	   \AND
       \name Luísa R. de A. Santos\email lrasantos@inf.ufrgs.br \\
       \name André G. Pereira \email agpereira@inf.ufrgs.br \\
       \addr Federal University of Rio Grande do Sul, Brazil}

\maketitle

\begin{abstract}
\textit{Goal Recognition} is the task by which an observer aims to discern the goals that correspond to plans that comply with the perceived behavior of subject agents given as a sequence of observations. 
Research on \textit{Goal Recognition as Planning} encompasses reasoning about the model of a planning task, the observations, and the goals using planning techniques, resulting in very efficient recognition approaches.
In this article, we design novel recognition approaches that rely on the \textit{Operator-Counting} framework, proposing new constraints, and analyze their constraints' properties both theoretically and empirically. 
The \textit{Operator-Counting} framework is a technique that efficiently computes heuristic estimates of \textit{cost-to-goal} using \textit{Integer/Linear Programming} (IP/LP). 
In the realm of theory, we prove that the new constraints provide lower bounds on the cost of plans that comply with observations. 
We also provide an extensive empirical evaluation to assess how the new constraints improve the quality of the solution, and we found that they are especially informed in deciding which goals are unlikely to be part of the solution.
Our novel recognition approaches have two pivotal advantages: first, they employ new IP/LP constraints for efficiently recognizing goals; second, we show how the new IP/LP constraints can improve the recognition of goals under both partial and noisy observability. 

\end{abstract}

\section{Introduction}

\textit{Goal Recognition} is the task by which an observer perceives an agent's behavior by an abstract \textit{sensor} as observations, and \textit{infers} a subset of goals that have plans that comply with the agent behavior~\cite{AIJ_SchmidtSG78} subject to domain model dynamics. 
In the most common setting for goal recognition, the observations are sub-sequences of operators, possibly partial and noisy, extracted from the original plan that the agent executed, which is unknown from the observer's perspective. 
In \textit{Goal Recognition as Planning}~\cite{ramirez2009plan}, such observations constitute constraints on which plans from the domain models explain valid goal hypotheses. 
Thus, the observer can apply \emph{planning} algorithms to reason about a \emph{goal recognition} task and select a subset of the possible goals.
An observer can then reason about a goal recognition task by \textit{planning} from the point of view of the observed agent and selecting a subset of the possible goals. 

Most approaches to \textit{Goal Recognition as Planning} often assume that the agent is rational, i.e., the most likely solution is the subset of goals that have plans complying with the observations with the least additional cost~\cite{ramirez2009plan,ramirez2010probabilistic,martin2015fast,sohrabi2016plan,masters2019JAIR}.
These approaches can be understood as \textit{cost-based} approaches, as they rely on comparing the costs of plans that either comply with or do not comply with the observations for the same goal. 
They select some subset of the goals by comparing these costs for all possible goals, making the efficiency of cost computation critical for such approaches. 

\textit{Cost-based} approaches often employ a search procedure of planning algorithms for such cost computation, making them accurately but computationally expensive~\cite{ramirez2009plan,ramirez2010probabilistic,sohrabi2016plan,vered2016_OnlineMirroring,masters2019JAIR}.
By contrast, other cost-based recognition approaches have relied on information from planning heuristics to speed up recognition without running a search procedure~\cite{martin2015fast,pereira2020landmarks,Santos2021}. 
The trade-off here is that the recognition process is limited to how informed the underlying heuristics are. 
The \textit{Operator-Counting} framework~\cite{pommerening2014lp} is a \textit{planning} technique that models a planning task as an \textit{Integer Program} (IP) model, typically relaxed into a \textit{Linear Program} (LP) model, in order to efficiently estimate to cost of solving it.
This framework combines the information of other planning techniques through IP/LP constraints, obtaining a more informed estimate of the cost of solving a \textit{planning task}.
Common sources of constraints are \textit{landmarks} \hlmc~\cite{hoffmann2004ordered}, the \textit{post-hoc} optimization \hpho~\cite{florian2013posthoc}, \textit{network flow} \hflow~\cite{bonet2014flow} and \textit{delete-relaxation} \hdr~\cite{Imai:2014bk}. 
Performing goal recognition using \textit{Operator-Counting} constraints offers two key advantages. 
First, the main advantage of using the operator-counting framework is that it enables one to manipulate the constraints directly to reason about the cost of solving planning tasks. 
Thus, like their planning heuristics counterparts, one can improve the informativeness of the goal recognition process by adding more constraints encoding information from the planning task.
Second, one can include constraints to deal specifically with the information available from goal recognition problems, including noise. 

Our previous work provides an efficient approach that uses the \textit{Operator-Counting} framework to solve \textit{goal recognition} tasks by including constraints based on the observations that restrict the set of solutions of the IP/LP model~\cite{Santos2021}. 
Within our recognition framework, we define a reference solution set to better compare the quality of solutions from earlier approaches. 
In this article, we consolidate our original work on goal recognition using the \textit{Operator-Counting} framework~\cite{Santos2021}, expanding the set of constraints and carefully studying the implication of the LP constraints to the recognition task. 
Thus, in this article, we provide a number of novel contributions to goal recognition in general, and to LP-based techniques in particular. 
First, we formally prove, in Section~\ref{sec:lp_goal_recognition}, that our original \textit{Operator-Counting} constraints provide lower bounds on the cost of plans that comply with the observations. 
Second, we introduce novel LP/IP \textit{landmark constraints}~\hlmc for goal recognition tasks that extract information \textit{on-the-fly}, from the observation sequence. 
Third, we develop further empirical experiments in Section~\ref{sec:experimental_setting} to improve the understanding of how the extra constraints impact the quality of the results. 
We show in Section~\ref{sec:evaluation} that the new constraints are more informed in deciding which goals are unlikely to be part of the solution. 
Our new experiments show that strengthening our recognition of heuristics leads the value of non-reference goals to increase more than the reference ones, thus better differentiating them.
In a broad context, our new approaches improve the quality of the solutions substantially and pave the way for exploring novel avenues of research for solving goal recognition tasks efficiently with high-quality accuracy.

\section{Background} 
In this section, we introduce the essential background on \textit{Planning} and \textit{heuristic functions}, as well as key concepts of the \textit{Operator-Counting} framework and its \textit{constraints}.

\subsection{Planning}
\label{sec:planning_task}

\textit{Planning} is the problem of finding a sequence of operators that achieves a particular goal condition from an initial state.
A {\sasplus}~\textit{planning task}~\cite{Backstrom:1993tc} is a tuple~$\planningtask = \tuple{\variables, \operators, \initialstate, \goalstate, \cost}$, where  $\variables$ is a set of discrete finite-domain \textit{variables}; $\operators$ is a set of~{\sasplus} \textit{operators}; $\initialstate$ is the \textit{initial state}; $\goalstate$ is a~{\sasplus} partial variable assignment denoting the \textit{goal condition}; and $\cost$ is a function that maps each operator in $\operators$ to a natural-valued cost. In this article, operators' costs are unitary.

An \textit{atom} is a pair $\tuple{\variable,v}$ of a variable $\variable \in \variables$, and one of its values $v \in \dom{\variable}$. 
The set of all atoms $\facts$ consists of all possible pairs of variables $\variable \in \variables$ with their respective possible values $v \in \dom{\variable}$. 
A \textit{partial state} is a set of atoms $\tuple{\variable,v}$ that mentions each variable at most once. 
Let $\vars(s)$ be the domain of variables in a partial state~$s$.
A variable assignment $s$ with $\vars(s) = \variables$ is a \textit{complete state} or simply a \textit{state}, and the set of all (complete) states over \variables~is the state-space~\states. 
A complete state~$s$ is \textit{consistent} with a (possibly partial) state $s'$ if $s' \subseteq s$.

Each \textit{operator} is a tuple $o = \tuple{p, e}$ of partial variable assignments, where $p = \pre(o)$ is the set of preconditions, and $e = \post(o)$ is the set of effects.
An operator $o$ is applicable in a state~$s$ if $\pre(o) \subseteq s$.
The state resulting from executing an applicable operator~$o$, denoted $s'= s\exec{o}$, is for all $\variable \in \variables$, $s'[V] =\post(o)[V]$ if $\variable\in\vars(\post(o))$ and $s'[V] =s[V]$ otherwise. 
Finally, an \textit{s-plan} \plan~is a sequence of operators $\tuple{o_1, o_2, \ldots, o_n}$ that starts at state $s$ and ends in a state that satisfies the condition \goalstate~after sequential applications of each operator $o_i \in \plan$. An \textit{s-plan} \plan~is \textit{optimal} if the value of $\cost(\plan)$ is minimal over all \textit{s-plan}, where $\cost(\plan) = \sum^{n}_{i=1} \cost(o_i)$. 
An \textit{s-plan} is called a \textit{plan} if $s = \initialstate$ which is a \textit{solution} for a planning task.

A \textit{heuristic function} for a given planning task $\planningtask$ with a set of states~$S$ is a function $\h : S \rightarrow \Real\cup\{\infty\}$ that estimates the cost of the solutions for a given state.
A given heuristic \h~is \textit{admissible} if and only if, for every state $s\in S$, $\h(s) \leq \hoptimal(s)$, where $\hoptimal(s)$ is the \textit{perfect heuristic}, i.e. the cost of an optimal solution for~$s$.

\subsection{Operator-Counting Framework}
\label{sec:oc_framework}

The \textit{Operator-Counting} framework combines the information from different sources through an \textit{Integer/Linear Program} ($\textup{IP/LP}$)~\cite{pommerening2014lp}. These sources provide constraints for a state~$s$ that must be satisfied by all $s$-plans. Common sources of constraints are state equation \hseq~\cite{bonet2013admissible}, landmarks \hlmc~\cite{hoffmann2004ordered}, the post-hoc optimization \hpho~\cite{florian2013posthoc}, the network flow \hflow~\cite{bonet2014flow}, and the delete relaxation \hdr~\cite{Imai:2014bk}. The framework has two main advantages: it allows us to efficiently combine the information a diverse set of sources maintaining admissibility, and it enables us to reason and to manipulate the information of the source directly.  
In general, the objective value of the linear program~($\textup{LP}$), a linear relaxation of the integer program, is used as heuristic function to guide the search.

This heuristic works by computing a \textit{pseudo-plan}. 
In the context of this article, a pseudo-plan is an unordered multiset of operators derived from the  \textit{operator-counts}. 
When ordered, a pseudo-plan may or may not yield a valid plan. 
Each operator must be executed at least the number of times indicated by its operator count, but in no particular order. 
\textit{Operator-counting variables} represent these operator counts in the IP/LP under \textit{operator-counting constraints} from Definition \ref{def:oc_constraint}.

\begin{definition}[Operator-Counting Constraint] Let $\planningtask$ be a planning task with operators $\operators$, and let $s$ be one of its states. 
Let $\mathcal{Y}$ be a set of real-valued and integer variables, including an operator-counting non-negative integer variable $\Y{o}$ for each operator $o \in \operators$.
A set of linear inequalities over $\lpvariables$ is an \textit{operator-counting constraint} for $s$ if for every valid $s$-plan~\plan, there exists a solution for it with $\Y{o} = \occur_{\plan}(o)$ for all $o\in\operators$ --- where $\occur_{\plan}(o)$ is the number of occurrences of operator~$o$ in the $s$-plan~\plan. 
We call $\constraints$ a set of operator counting constraints. 
\label{def:oc_constraint}
\end{definition}

Operator counting variables then allow one to define heuristics computed indirectly via an integer/linear program that enforces constraints computed from the structure of a planning task. 
We formalize the broad class of such heuristics in Definition~\ref{def:oc_heuristic}.

\begin{definition}[Operator-Counting $\textup{IP}$/$\textup{LP}$ Heuristic]
    \label{def:oc_heuristic}
The \textit{operator-counting integer program} $\textup{IP}^C$ for a set of operator-counting constraints $\constraints$ for state $s$ is
\begin{align*}
    & \text{minimize} \sum_{o\in O} \cost(o)\Y{o}\\
    & \text{subject to $C$},\\
    & \Y{o}\in\Z^{+}_0.
\end{align*}
The $\textup{IP}$ \textit{heuristic} $h^{\textup{IP}}$ is the objective value of $\textup{IP}^C$, and the $\textup{LP}$ \textit{heuristic} $h^{\textup{LP}}$ is the objective value of its linear relaxation. If the $\textup{IP}$ or $\textup{LP}$ is infeasible, the heuristic estimate is $\infty$.
\end{definition} 

An important result from \citeauthor{pommerening2014lp}~\citeyear{pommerening2014lp}, is that, as we add more constraints into the operator-counting framework, we exclude solutions that are inconsistent with solutions to the underlying planning problem. 
The practical result is that operator-counting heuristics can only get stronger as we add further constraints to them, a property which we replicate in Proposition~\ref{prop:dominance}. 

\begin{proposition}[Dominance]
    \label{prop:dominance}
Let $C$, $C'$ be functions that map states~$s$ of a planning task $\planningtask$ to constraint sets for $s$ such that $C(s) \subset C'(s)$ for each $s$. Then the heuristic~$h^{\textup{IP/LP}}$ for $C'$ \textit{dominates} the respective heuristic for $C$: $h^{IP/LP}_{C} \leq h^{IP/LP}_{C'}$.
\end{proposition}

The key sources of operator-counting constraints investigated in this article stem from the various types of \textit{landmarks}~\cite{hoffmann2004ordered}. 
We follow the formalization according to \citet{pommerening2014lp} and cite the original formalization when appropriate. 
A \textit{disjunctive action landmark}~$L$ for a state~$s$ is a set of operators, where at least one operator in~$L$ must be part of every $s$-plan for state~$s$~\cite{hoffmann2004ordered}. 
Definition~\ref{def:lm_constraints} encodes the landmark information as an operator-counting constraint. 
Since at least one operator in a disjunctive action landmark must be used, the sum of their respective operator-counting variables should be at least one.

\begin{definition}[Landmark Constraint]
Let $L\in\operators$ be a disjunctive action landmark for a state~$s$ of a planning task~$\Pi$. The \textit{landmark constraint} for~$L$ is,
\begin{align*}
     \sum_{o\in L} \Y{o} \geq 1.
\end{align*}
\label{def:lm_constraints}
\end{definition}

Most of the efficient algorithms to compute landmark constraints are, in general, correct and incomplete~\cite{hoffmann2004ordered}. Thus, these algorithms only produce and extract a subset of the landmarks of a planning task $\planningtask$. 

\section{Goal Recognition as Planning}
\label{sec:goal_recognition}

In this section, we introduce the task of \textit{Goal Recognition}, following the formalism of \textit{Goal Recognition as Planning}~\cite{ramirez2009plan}. 
From the definition of a \textit{goal recognition task}, we introduce a specific definition of solution that enables a precise comparison between different approaches that solve the task.

Definition~\ref{def:gr} formally defines the task of goal recognition using a modified version of a planning task $\planningtask$ (Section~\ref{sec:planning_task}). 

\begin{definition}[Goal Recognition Task]
\label{def:gr}
    Let $\grplanningtask$ be a planning task without a goal condition, $\goalconditions$ be a set of goal conditions (or hypotheses), and $\observations=\tuple{\obs{o}_1,\ldots,\obs{o}_{m}}$ be a sequence of observations. 
    The sequence of observations~$\observations$ is extracted from a plan~$\plan=\tuple{o_1,\ldots,o_{n}}$ of the planning task \grplanningtask{} with the goal condition~$\rgoal\in\goalconditions$.
    Then, a \textit{goal recognition} task is a tuple $\grtask = \tuple{\grplanningtask, \goalconditions, \observations}$.
\end{definition}
This bakes into the definition the assumption that the set of goal condition hypotheses $\goalconditions$ does indeed contain a goal $\rgoal\in\goalconditions$ that explains $\observations$. 
Thus, by construction there must be a plan for \grplanningtask{} with the goal condition \rgoal{}. 
As we further refine the definition of the elements of goal recognition tasks, we introduce terminology to refer to these components. 
First, we refer to the goal condition \rgoal{} as the \textit{intended reference goal}. 
Second, we represent the label extracted from an operator~$o$ as observation~$\obs{o}$.
The label assigned to observation~$\obs{o}$ is the same label assigned to operator~$o$. 
Thus, we access observations and operators using the labels interchangeably.
Finally, as we refine observations in Definitions~\ref{def:comply}-\ref{def:observations_noisy}, we introduce the notion of noise in the observations. 

Wherever we introduce noise in the observations, we can extend the definition of the goal recognition task with an assumption of the maximal level of noise in the observations. 
In these cases, we extend Definition~\ref{def:gr} with the level of noise or unreliability $\unreliability$ of the observations. 
The problem under noise then becomes a quadruple $\tuple{\grplanningtask, \goalconditions, \observations, \unreliability}$ so that Definition~\ref{def:gr} generalizes to $\tuple{\grplanningtask, \goalconditions, \observations} = \tuple{\grplanningtask, \goalconditions, \observations, \unreliability = 0}$.

\begin{definition}[Observation Compliance]
\label{def:comply}
Let~$\tuple{\grplanningtask, \goalconditions, \observations}$ be a plan recognition task, and let $\plan=\tuple{o_1, \ldots, o_n}$ be a plan for a planning task \grplanningtask{} with the goal condition~$\goalstate\in\goalconditions$ and a sequence of observations $\observations=\tuple{\obs{o}_1, \ldots, \obs{o}_m}$.
Plan $\plan$ \textit{complies} with $\observations$ if there is a monotonic function $f : [1,m] \mapsto [1,n]$ that maps all labels of operator indexes in $\observations$ to indexes in $\plan$, such that $ \obs{o}_i = o_{f(i)}$, i.e., the labels match.
\end{definition}

\textit{Observations} are the key evidence that one can use to infer the intended reference goal in goal recognition tasks. 
The trivial problem setting for an algorithm performing goal recognition involves access to the full sequence of observations of a plan. 
However, in most realistic settings, observations suffer from a number of flaws, stemming from both the erratic behavior of the agent under observation and the sensing capabilities of the recognizer. 
Thus, we define three classes of sequences of observations: \textit{optimal} and \textit{suboptimal} (Definition~\ref{def:observations_opt_subopt}) observations, and \textit{noisy} observations (Definition~\ref{def:observations_noisy}).

\begin{definition}[Sequence of Observations] \label{def:observations_opt_subopt}
Let~$\plan=\tuple{o_1,\ldots,o_n}$ be a plan for the planning task \grplanningtask{} with with the reference goal condition \rgoal{}.
Then, a \textit{sequence of observations} \observations{} is a sequence of labels of operators extracted from the plan \plan{} maintaining their relative order. 
The sequence may be partial, containing any number of operator labels from the plan \plan.
An \textit{optimal} sequence of observations is extracted from an optimal plan and a \textit{suboptimal} sequence of observations is extracted from a suboptimal plan. An optimal/suboptimal observation is part of an optimal/suboptimal sequence of observations. 
\end{definition}

\begin{definition}[Noisy Observations]
\label{def:observations_noisy}

Let~$\observations$ be a sequence of observations extracted from $\plan$. 
Then, $\observations'$ %
is a \textit{noisy} sequence of observations if it is equal to $\observations$ with the addition of at least one observation in any place of the original sequence of an operator from~$\operators - \plan$.

\end{definition}

We extend the standard definition from~\citet{ramirez2009plan} of an \textit{exact solution set} for a goal recognition task to also consider suboptimal observation sequences (Definition~\ref{def:gr_solutionset}) and call it \textit{reference solution set}. 
We define the \textit{reference solution set} as a subset of the goal conditions such that there exists a complying plan as suboptimal as or less than the plan that generated the observations for the reference goal. 

\begin{definition}[Reference Solution Set]
\label{def:gr_solutionset}
Let~$\langle \grplanningtask, \goalconditions, \observations\rangle$ be a \textit{goal recognition task} and  $\rgoal\in\goalconditions$ the reference goal.
Let \plan{} be the plan for \grplanningtask{} with the reference goal \rgoal{}, from which \observations{} is extracted. 
Let \hostarPi{\goalstate}{} be a heuristic that returns the cost of an optimal plan from a state for \grplanningtask{} with goal condition \goalstate{} restricted to the set of plans that comply with \observations. 
Let \hstarPi{\goalstate} be a heuristic that returns the cost of an optimal plan for \grplanningtask{} with goal condition~$\goalstate$. 
Heuristics \hostarPi{\goalstate} and \hstarPi{\goalstate} are equal to $\infty$ if no plan exists.
Then, the \textit{reference solution set} for the \textit{goal recognition task} is
	$$\grsolution = \{\goalstate \in \goalconditions \mid   
	   \frac{\hostarPi{\goalstate}(\initialstate)}{\hstarPi{\goalstate}(\initialstate)} \leq \frac{\cost(\plan)}{\hstarPi{\rgoal}(\initialstate)} \land \hostarPi{\goalstate}(\initialstate)\neq\infty \}$$
\end{definition}

\begin{figure}[ht]
    \centering
    \includegraphics[scale=0.7]{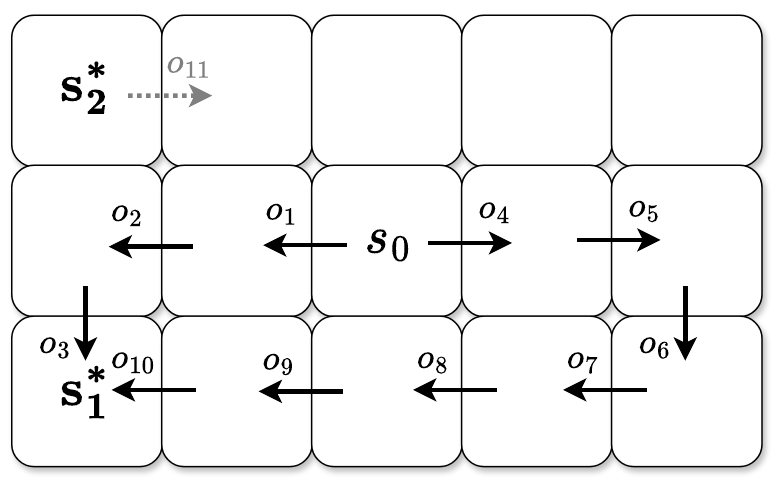}
    \caption{A goal recognition task example.}
    \label{fig:ex1}
\end{figure}

\begin{example}
Figure~\ref{fig:ex1} illustrates a goal recognition task~$\langle \grplanningtask, \{\goalstate_{1}, \goalstate_{2}\}, \observations\rangle$.  %
Let~$\goalstate_1$ be the reference goal \rgoal{}. 
Now consider different sequences of observations. $\observations_1 = \tuple{ \obs{o}_1}$ is an optimal sequence of observations because it is extracted from the optimal plan~$\plan_1=\tuple{o_1,o_2,o_3}$, $\observations_2 = \tuple{ \obs{o}_5,  \obs{o}_7,  \obs{o}_9}$ and $\observations_3 = \tuple{ \obs{o}_4,  \ldots,  \obs{o}_{10}}$ are suboptimal sequences of observations because they are extracted from the suboptimal plan~$\pi_2=\tuple{o_4,\ldots,o_{10}}$, and $\observations_4 = \tuple{ \obs{o}_4,\ldots,\obs{o}_{10}, \obs{o}_{11}}$ is a suboptimal and noisy sequence of observations because it was extracted from~$\plan_2$ and the observation of~$\obs{o}_{11}$ was added. 
We compute the reference solution set for goal recognition tasks with noisy observations by ignoring noisy observations in the sequence of observations. 
The reference solution set for any of these observation sequences is $\Gamma^*=\{\goalstate_1\}$. 
For example, $\hostarPi[\observations_4]{\goalstate_1}/\hstarPi{\goalstate_1}=7/3$, $\hostarPi[\observations_4]{\goalstate_2}/\hstarPi{\goalstate_2}=9/3$, $\cost(\plan_2)/\hstarPi{\rgoal}=7/3$
and thus $\Gamma^*=\{\goalstate_1\}$.
\end{example}

\paragraph{Cost Difference-Based Goal Recognition}
\citeauthor{ramirez2010probabilistic}~\citeyear{ramirez2009plan,ramirez2010probabilistic} address goal recognition tasks where agents can pursue their goal suboptimally\footnote{\citeauthor{ramirez2010probabilistic}~\citeyear{ramirez2010probabilistic} present a different formulation, but here we present the formulation from \citeauthor{masters2019JAIR}~\citeyear{masters2019JAIR}, which is effectively the same.}.
For this, they consider the cost difference between an optimal plan that complies with the observations and the cost of an optimal plan that does not comply with at least one observation for each goal condition. 
Similarly, \citeauthor{Vered_2017gj}~\citeyear{Vered_2017gj} empirically show that the cost ration between an optimal plan and the cost of an optimal plan that complies with the observations for each goal condition can rank goal hypotheses similarly to \citeauthor{ramirez2010probabilistic}~\citeyear{ramirez2009plan,ramirez2010probabilistic}. 
The subset of goal conditions with minimum such cost differential is part of their solution. 
The definition below formalizes this idea:

\begin{definition}[Cost Difference Solution Set] 
\label{def:gr_solutionset_cost_diff}
Let~$\langle \grplanningtask, \goalconditions, \observations\rangle$ be a \textit{goal recognition task} and  $\rgoal\in\goalconditions$ the reference goal.
Let \hostarPi{\goalstate}{} be a heuristic that returns the cost of an optimal plan from a state for \grplanningtask{} with goal condition \goalstate{} restricted to the set of plans that comply with \observations.
Let \hstarPi{\goalstate} be a heuristic that returns the cost of an optimal plan for \grplanningtask{} with goal condition~$\goalstate$. 
The minimal difference between observation-complying plans and optimal plans $\delta_{\text{min}}$ for hypotheses $\goalconditions$ is:
\begin{equation}
\delta_{\text{min}} = \min_{\goalstate_i\in\goalconditions~:~\hostarPi{\goalstate_i}(\initialstate) < \infty} \{\hostarPi{\goalstate_i}(\initialstate) - \hstarPi{\goalstate}(\initialstate)\}
\label{equation-cost-delta}
\end{equation}
Then, the \emph{Cost Difference Solution Set} $\Gamma^{\text{Cost-Diff}}$ of goals hypotheses that are consistent with minimal cost difference $\delta_{\text{min}}$ is:
\begin{equation}
\Gamma^{\text{Cost-Diff}} = \{\goalstate_i \in \goalconditions \mid \hostarPi{\goalstate_i}(\initialstate) - \hstarPi{\goalstate_i}(\initialstate) = \delta_{\text{min}}\}
\label{equation-cost-solution-set}
\end{equation}
\end{definition}

\begin{example}
Figure~\ref{fig:ex2} illustrates a goal recognition task~$\langle \planningtask_{\textup{P}}, \{\goalstate_{1}, \goalstate_{2}\}, \{\obs{o}_{1}\}\rangle$. 
Let~$\goalstate_1$ be the reference goal \rgoal{}, and $\obs{o}_{1}$ be the only observation of the task. 
For this example, the cost of  optimal complying plans are~$\hostarPi{\goalstate_1}(\initialstate)=4$ and $\hostarPi{\goalstate_2}(\initialstate)=3$. 
Thus, using only this information for observations extracted from suboptimal behavior one would return solution~$\Gamma^*=\{\goalstate_2\}$, which intuitively is not the solution. 
However, by using Definition~\ref{def:gr_solutionset_cost_diff} to compute the solution we have $\hostarPi{\goalstate_1}(\initialstate) - \hstarPi{\goalstate_1}(\initialstate)=0$ and  
$\hostarPi{\goalstate_2}(\initialstate) - \hstarPi{\goalstate_2}(\initialstate)=2$, which is the intuitive solution.
\end{example}

\begin{figure}[ht]
    \centering
    \includegraphics[scale=0.7]{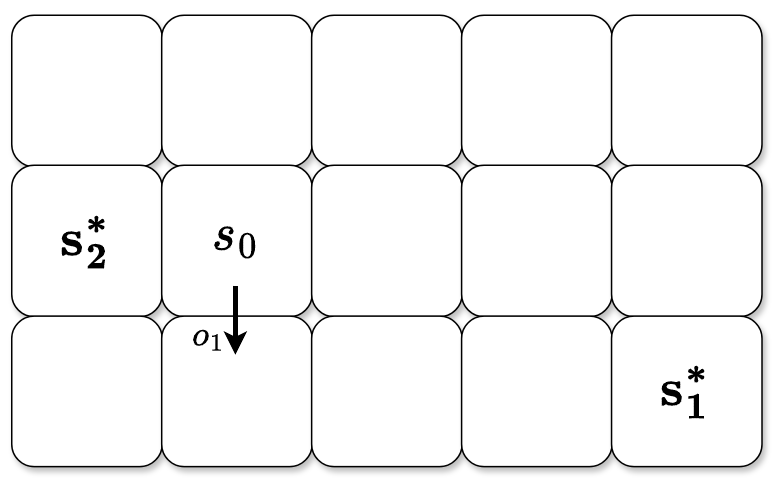}
    \caption{A goal recognition task example considering cost differences.}
    \label{fig:ex2}
\end{figure}

The main limitation of this approach is that computing the cost of $\hostarPi{\goalstate}(\initialstate)$ and $\hstarPi{\goalstate}(\initialstate)$ for \sasplus{} planning tasks is a PSPACE-complete problem, and very costly in practice, as we show in Proposition~\ref{prop:gr_pspace}. 
Thus, most approaches rely on heuristics to estimate these costs. 
In this article, we introduce and study a set of heuristics to estimate these values based on \textit{Integer/Linear} programs.

\begin{proposition}[Computing $\Gamma^{\text{Cost-Diff}}$ is PSPACE-complete]%
    \label{prop:gr_pspace}
    Let~$\grtask = \langle \grplanningtask, \goalconditions, \observations\rangle$ be a \textit{goal recognition task},
    computing the reference-solution set $\Gamma^{\text{Cost-Diff}}$ for~$\grtask$ is PSPACE-complete. 
\end{proposition}
\begin{proof}
    Our proof hinges on showing the upper and lower bound complexity for computing each of the components of the set from Definition~\ref{def:gr_solutionset_cost_diff}, specifically the cost of computing $\hstarPi{\goalstate}(\initialstate)$ and $\hostarPi{\goalstate}(\initialstate)$ for each goal hypothesis in $\grtask$. 
    First, the cost of computing the cost $\hstarPi{\goalstate}(\initialstate)$ of reaching each goal hypothesis $\goalstate$ has a trivial upper and lower bound. 
    This comes from the complexity of solving the corresponding planning task, proven to be PSPACE-complete by \citet{Backstrom:1993tc}.
    Second, the cost of computing $\hostarPi{\goalstate}(\initialstate)$ is no worse than PSPACE-complete, since we can use the reduction of \citet{ramirez2009plan} for computing an observation-complying plan to a regular planning problem. 
    Since both components of the overall computation are no more expensive than PSPACE-complete, and at least one of the components is at least PSPACE-complete, we can conclude that Goal Recognition as Planning is, indeed, PSPACE-complete.
\end{proof}

\section{Linear Programming-Based Goal Recognition}
\label{sec:lp_goal_recognition}

In this section, we develop an IP/LP constraint to produce a lower bound on the cost of an optimal complying plan.
Key to this approach is the addition of \textit{observation-counting constraints} that ensure that the IP/LP only computes solutions that satisfy the observations counts. 
We use the operator-counting framework to improve our lower bounds since it restricts the set of solutions of the IP/LP to those that satisfy the operator-counting constraints. 

Definitions~\ref{def:complying-constraints} and~\ref{def:complying-ip} formally introduce the set of \textit{observation-counting constraints} and the IP/LP that ensures that the solution computed satisfies all observation counts.

\begin{definition}[Observation-Counting Constraints]
\label{def:complying-constraints}
Let $\tuple{\planningtask_{\textup{P}}, \goalconditions, \observations}$ be a \textit{goal recognition task} with operators \operators.
Let~$\mathcal{Y}$ be the set of \textit{operator-counting variables} for $\planningtask_{\textup{P}}$ with a variable~$\Y{o}$ for each operator $o \in \operators$, let~$\mathcal{Y}^{\observations}$ be a set of non-negative integer variables with a variable~$\Yobs{o}$ for each operator $o \in \operators$, and let~$\unreliability$ be the \textit{unreliability rating} of the sensor of observations.
The set of \textit{observation-counting constraints} \constraintso~consists of:
\begin{align}
    & \Yobs{o}\leq \occur_{\observations}(o)                                 & \text{for all } o \in \operators \label{eq:obs1} \\
    & \Yobs{o}\leq \Y{o}                                                 & \text{for all } o \in \operators \label{eq:obs2} \\ 
    & \sum_{\Yobs{o} \in \mathcal{Y}^{\observations}} \Yobs{o} \geq |\observations| -  \lfloor|\observations| * \unreliability\rfloor & \label{eq:obs3}\\
    & \Y{o},\Yobs{o}\in\Z^{+}_0. \nonumber
\end{align}
\end{definition}

In the IP/LP, the set of constrains~(\ref{eq:obs1}) limits the value of each~$\Yobs{o}$ by the number of occurrences of the operator~$o$ in $\observations$.
Next, the set of constraints~(\ref{eq:obs2}) binds the two sets~$\Yobs{o}$ and~$\Y{o}$ of variables.
This set of constraints guarantees that~$\Y{o}$ acts as an upper bound for~$\Yobs{o}$. 
Thus, to increase the count of~$\Yobs{o}$ the IP/LP solution must first increase the count of~$\Y{o}$ which is minimized in the objective function and can be restricted by the set of operator-counting constraints $\constraints$. 
Finally, constraint~(\ref{eq:obs3}) enforces the satisfaction of a subset of observations, since each~$\Yobs{o}$ is limited by the number of times~$o$ appears in $\observations$. 
$\unreliability$ is the \textit{unreliability rating} of the sensor that represents the expected percentage of mistaken observations. 
For now, consider that the sensor is perfect and $\unreliability$ is equal to zero.

\begin{definition}[Satisfying IP/LP Heuristic Function]
\label{def:complying-ip}
Let $\tuple{\planningtask_{\textup{P}}, \goalconditions, \observations}$ be a \textit{goal recognition task}, $\goalstate$ be one of the goal conditions of the task, and~$s$ be a state of $\planningtask_{\textup{P}}$.
Let \constraints{} be a set of operator-counting constraints for state~$s$ of planning task~$\planningtask_{\textup{P}}$ with goal condition~$\goalstate$, and $\constraintso{}$ be a set of observation-counting constraints for the goal recognition task.
Then, the \textit{satisfying} $\textup{IP/LP}$ heuristic function for the sets of constraints $\constraints$, and $\constraintso{}$ is,
\begin{align}
    \text{minimize}~\sum_{o\in \operators} \cost(o)\Y{o},~\text{subject to \constraints~and \constraintso}.\nonumber 
\end{align}

The \textit{satisfying} $\textup{IP}$ \textit{heuristic function} $\hoipPi{\goalstate}(s)$ is the objective value of the $\textup{IP}$, and the \textit{satisfying} $\textup{LP}$ \textit{heuristic function} $\hoPi{\goalstate}(s)$~is the objective value of its linear relaxation. 
If the $\textup{IP}$ or $\textup{LP}$ is infeasible, the heuristic estimate is $\infty$.
\end{definition} 

\begin{theorem}[\hoipPi{\goalstate} is a Lower Bound] The value of $\hoipPi{\goalstate}(s)$~is a lower bound on the cost of an optimal complying plan.
\label{theo:hip}
\end{theorem}
\begin{proof}
    In order to prove this property, we must show that any optimal complying plan satisfies Constraints~\ref{eq:obs1}-\ref{eq:obs3}, and that the solution to the objective cost of the resulting $\textup{LP}$ is a lower bound. 
To show this, let $\pi = \tuple{o_1, \dots, o_n}$ be an optimal complying $s$-plan for task \grplanningtask{} with goal condition $\goalstate$, and $s$ a state from $\grplanningtask{}$. 
First, since \constraints{} is a set of operator-counting constraints for state~$s$ of task~$\grplanningtask{}$ with goal condition~$\goalstate$, this set of constraints~$\constraints{}$ is satisfied by setting each~$\Y{o}$ equal to $\occur_{\pi}(o)$.

Second, $\constraintso$ is defined by the sequence of observations $\observations$. Thus, we can set each~$\Yobs{o}$ to $\occur_{\observations}(o)$. Using this attribution we satisfy constraints sets \ref{eq:obs1} and \ref{eq:obs3}.

Third, we need to show that the set of constraints \ref{eq:obs2} is satisfied. 
Since~$\pi$ is a complying plan it guarantees that there is a monotonic function~$f$ that maps all labels of operator indexes in $\observations$ to indexes in $\plan$, such that $\obs{o}_i = o_{f(i)}$. 
Thus, for each label~$\obs{o} \in \observations$ there is a label~$o \in \pi$ which guarantees~$\Yobs{o}\leq \Y{o}$ for all $o \in \mathcal{O}$.

These three points combined guarantee that the proposed attribution is a solution for the IP with $\constraints{}$ and $\constraintso$. 

Finally, the cost of optimal solution of the IP/LP can not exceed the cost of any specific solution, thus,

\begin{align}
    \hoPi{\goalstate}(s) \leq \hoipPi{\goalstate}(s) \leq \cost(\pi) = \hostarPi{\goalstate}(s).
\end{align}
\end{proof}

\paragraph{Computing the Solution Set.} Having defined $\hoPi{\goalstate}$ we use Definition~\ref{equation-cost-solution-set} to compute our solution set. 
We use $\hoPi{\goalstate}$ to estimate a lower bound of $\hostarPi{\goalstate}$, and $h_{s^*}$ to estimate a lower bound of $\hstarPi{\goalstate}$. 
For clarity, we repeat below the equations of Definition~\ref{equation-cost-solution-set} that compute our solution set $\dhc{}$ with the heuristics that estimate the lower bounds.%

\begin{equation}
\delta_{\text{min}} = \min_{\goalstate\in\goalconditions\text{ }:\text{ }\hoPi{\goalstate}(\initialstate) < \infty} \{\hoPi{\goalstate}(\initialstate) - \hPi{\goalstate}(\initialstate)\}
\label{equation-delta}
\end{equation}
\begin{equation}
\dhc = \{\goalstate \in \goalconditions \mid \hoPi{\goalstate}(\initialstate) - \hPi{\goalstate}(\initialstate) = \delta_{\text{min}}\}
\label{equation-solution-set}
\end{equation}

\paragraph{Noisy Observations.}
\label{sec:addressing_noisy_observations}

In most realistic settings, unreliable sensors may add noisy observations to the sequence of observations. 
Consider a goal recognition task in our running example with $\observations = \tuple{\obs{o}_4,\ldots,\obs{o}_{10},\obs{o}_{11}}$. 
Then, $\hoPi{\goalstate_1} = 13$ and $\hoPi{\goalstate_2} = 11$. 
In this situation we would have $\delta_{\text{min}} = 8$, and $\Gamma^\textup{LP} = \{\goalstate_2\}$. However, the observation~$~\obs{o}_{11}$ is unlikely to be part of any plan that generates the sequence of observations for either of the two goals. 

Our approach addresses noisy observations through the $\unreliability$ parameter (from the constraint of Eq~(\ref{eq:obs3})) which is the \textit{unreliability rating} of the sensor that represents the expected percentage of mistaken observations.  
The unreliability rating requires that at least $|\observations| -  \lfloor|\observations| * \unreliability\rfloor$ observations be satisfied by the solution found. 
If $\unreliability = 0$, all observations must be satisfied, whereas if $0 < \unreliability < 1$, some observations can be ignored in order to minimize the objective value of~$h_{\observations}$ for each goal candidate. 
Consider our example from Figure~\ref{fig:ex1} with $\observations = \tuple{ o_4,\ldots,o_{10},o_{11}}$ and $\unreliability = 0.2$. 
Here, the integer program~$\textup{IP}$ has to satisfy~$7$ observations, so $\hoPi{\goalstate_1} = 7$ and $\hoPi{\goalstate_2} = 9$. 
Recall that $\hstarPi{\goalstate_1}=3$ and $\hstarPi{\goalstate_2} = 3$. 
In this situation we have $\delta_{\text{min}} = 4$, and~$\dhc = \{\goalstate_1\}$.

\begin{theorem}[\hoipPi{\goalstate} is a Lower Bound in the Presence of Noisy Observations] If the \textit{unreliability rating} $\unreliability$ of the sensor is greater than zero and there at most $\lfloor|\observations| * \unreliability\rfloor$ noisy observations in $\observations{}$, then the value of $\hoipPi{\goalstate}(s)$~is a lower bound on the cost of an optimal complying plan.
\label{theo:hipnoisy}
\end{theorem}
\begin{proof}
Let $\pi = \tuple{o_1, \dots, o_n}$ be an optimal complying $s$-plan for the non-noisy observation in $\observations{}$ for task $\grplanningtask{}$ with goal condition $\goalstate$, and $s$ a state from $\grplanningtask{}$. 

First, since $\pi$ is a valid plan we can satisfy $\constraints{}$ by setting each~$\Y{o}$ equal to $\occur_{\pi}(o)$ as before.

Second, note that if $\obs{o}$ is a noisy observation, then it must satisfy property~$\obs{o} \in \mathcal{O} - \pi$ according to Definition~\ref{def:observations_noisy}.
Thus, we can set each~$\Yobs{o}$ to $\occur_{\observations}(o)$ if $\obs{o} \in \pi$ and to zero otherwise. 
Using this attribution and the fact that~$\pi$ is an optimal complying $s$-plan for the non-noisy observations, we satisfy constraints sets \ref{eq:obs1} and \ref{eq:obs2}.

Third, we need to show that the set of constraint \ref{eq:obs3} is satisfied. 
Since there are at most $\lfloor|\observations| * \unreliability\rfloor$ noisy observations in $\observations{}$, then:

\begin{align*}
    \sum_{\Yobs{o} \in  \mathcal{Y}^{\observations} : o \in \pi}\Yobs{o} - \sum_{\Yobs{o} \in  \mathcal{Y}^{\observations} : o \in \mathcal{O} - \pi} 
    \Yobs{o}
   & = \sum_{o \in \pi} \occur_{\observations}(o)  \\
   & = \sum_{o \in \mathcal{O}} \occur_{\observations}(o) - \sum_{o \in \mathcal{O}-\pi} \occur_{\observations}(o)    \geq
    |\observations| -  \lfloor|\observations| * \unreliability\rfloor
\end{align*}

As before, these three points combined guarantee that the proposed attribution is a solution for the IP with $\constraints{}$ and $\constraintso$, and the cost of optimal solution of the IP/LP cannot exceed the cost of any specific solution.
\end{proof}

We now have observation-counting constraints within the framework of operator-counting that provide provably correct lower bounds to plans that comply with constraints. 
These constraints allow us to approximate the cost difference solution set from Definition~\ref{def:gr_solutionset_cost_diff} within the operator counting framework as an LP. 
However, the key advantage of our LP framework for goal recognition is that we can add further constraints (either from the Operator Counting Framework, or otherwise) and strengthen our estimates. 
Indeed, much like the operator counting framework, as we further constrain the LP to better match the set of operators included in real plans to solve a planning task, our estimates get closer to the exact computation of the solution set from \citeauthor{ramirez2009plan}~\citeyear{ramirez2009plan}. 
\section{Novel LP Constraints for Goal Recognition}
\label{sec:adapting_oc}

We now develop novel \textit{Operator-Counting constraints} specifically for goal recognition tasks. 
This is an entirely novel framework of \textit{Linear Programming} constraints that combines the traditional operator-counting constraints from Definition~\ref{def:oc_constraint} and the observation-counting constraints from Definition~\ref{def:complying-constraints}. 
Intuitively, a constraint for the operator-counting framework is an operator-counting constraint if every $s$-plan satisfies it. 
In other words, operator-counting constraints induce a linear program whose solutions reflect the space of actions required by their corresponding planning problem. 
Definition~\ref{def:oc_constraint_grt} extends this property to goal recognition problems. 
It formally states that a constraint for the operator-counting framework is an operator-counting constraint for a goal recognition task if every \emph{observation complying} $s$-plan satisfies it. 

\begin{definition}[Operator-Counting Constraints for a Goal Recognition Task] Let $\tuple{\planningtask_{\textup{P}}, \goalconditions, \observations}$ be a goal recognition task, 
let the planning task~$\planningtask$ resulting from the union of the task~$\planningtask_{\textup{P}}$ with a goal condition $\goalstate\in\goalconditions\cup\{ \pre(\obs{o}) | \obs{o} \in  \observations \}$. 
Let $s$ be one of the states of task~$\planningtask$, and $\mathcal{Y}$ be a set of real-valued and integer variables, including an operator-counting non-negative integer variable $\Y{o}$ for each operator $o \in \operators$.
A set of linear inequalities over $\mathcal{Y}$ is an \emph{operator-counting constraint for a goal recognition task} for the state~$s$ if for every valid complying $s$-plan~\plan{} of task~$\planningtask$, there exists a solution for the set of linear inequalities with $\Y{o} = \occur_{\plan}(o)$ for all $o\in\operators$.
\label{def:oc_constraint_grt}
\end{definition}

In what follows, we introduce new landmark constraints that use the observations to derive further landmarks tied to necessary conditions in plans that induce such observations. 
Key to these new constraints is their compatibility with noisy observations. 

Much like previous work on computationally efficient approaches for goal recognition \cite{pereira2020landmarks}, our IP formulation approximates the set of goals that are compatible with the observations without computing full plans like in \citet{ramirez2009plan}. 
Unlike such previous work, however, the operator-counting framework allows us to tighten our approximation by adding further constraints. 
In theory, if we could add enough constraints to obtain the perfect heuristic for the underlying planning problem for a goal hypothesis (and the perfect heuristic for the plans that comply with the observations) then we have RGs optimal recognition.

\label{sec:lmc_improved}

Recall the concept of disjunctive landmarks of~\citeauthor{Porteous:2001ud}~\citeyear{Porteous:2001ud}. 
Let~$s$ be a state of the task~$\planningtask_{\textup{P}}$ with a goal condition $\goalstate\in\goalconditions$. 
Then, a \emph{disjunctive action landmark for a goal recognition task}~$L$ for the state~$s$ is a set of operators such that at least one operator in~$L$ must be part of every complying $s$-plan. 
In the context of goal recognition, such landmarks have two key properties. 
First, when restricted to non-noisy observations, every operator part of the sequence of observations~$\observations$ is an action landmark for the goal recognition task. 
This means that the action corresponding to every observation must be part of any valid pseudo-plan for this task.
Second, the preconditions of these operators are sub-goals that must be satisfied by every complying plan. 
These subgoals are in addition to the goal condition that must be satisfied by every plan. 
Thus, we can use traditional landmark extraction techniques to generate new landmarks using the preconditions of observations. 
Definition~\ref{def:lm_constraints_goal_recognition} encodes this new type of constraint.

\begin{definition}[Landmark Constraint for Goal Recognition Tasks]
Let $L\in\operators$ be a disjunctive action landmark for a state~$s$ of a planning task~$\planningtask_{\textup{P}}$ with goal condition defined as $pre(\obs{o})$ with $\obs{o}\in\observations$. The \emph{constraint} for~$L$ is,
\begin{align}
    \sum_{o\in L} \Y{o} \geq [\Yobs{o} > 0].
\end{align}
We call the set of all such constraints $\lmc_\observations$. 
\label{def:lm_constraints_goal_recognition}
\end{definition}

The main difference between Definition~\ref{def:lm_constraints_goal_recognition} and traditional landmark constraints is that the bound of the constraint is equal to $[\Yobs{o} > 0]$ instead of~$1$ from Definition~\ref{def:lm_constraints}. 
This means that $\Y{o}$ will be one if we observe operator $o$, and zero otherwise. 
Since this formulation depends on variables $\Yobs{o}$ from Definition~\ref{def:complying-constraints}, constraints from the new Definition~\ref{def:lm_constraints_goal_recognition} must be used with our observation-counting constraints~$\constraintso$. 
In practice, we implement these constraints as: 
\begin{align*}
    0 \leq \sum_{o\in L} \Y{o} - \frac{\Yobs{o}}{\occur_{\observations}(o)} \leq \infty.
\end{align*}
To emulate the ``binary'' variable of the test on the corresponding observation constraint. 
This means that the value we are comparing against $\Y{o}$ depends on what operators we have in the observations.
In turn, this formulation means that the solver only needs to enforce constraints of an observation~$\Yobs{o}$ if the value of~$\Yobs{o}$ is greater than~$0$. 
This enables the $\textup{IP/LP}$ to address noisy observations automatically. 
Having defined landmark constraints for goal recognition tasks we can define the new~\holmcsg heuristic. 

\begin{definition}[Landmark Heuristic for Goal Recognition Tasks]
\label{def:lm-heuristic}
Let~$\lmc_\observations$ be set of constraints of \emph{landmark constraint for goal recognition tasks} computed using as goal conditions~$pre(\obs{o})$ of each~$\obs{o}\in\observations$. 
Then, the \holmcsg{} heuristic extends Definition~\ref{def:complying-ip} with the set of constraints $\lmc_\observations$, i.e., the $\textup{IP/LP}$ of~\holmcsg{} includes constraints \constraints, \constraintso, and $\lmc_\observations$. 
\end{definition} 

\begin{figure}[ht]
    \centering
    \includegraphics[scale=0.7]{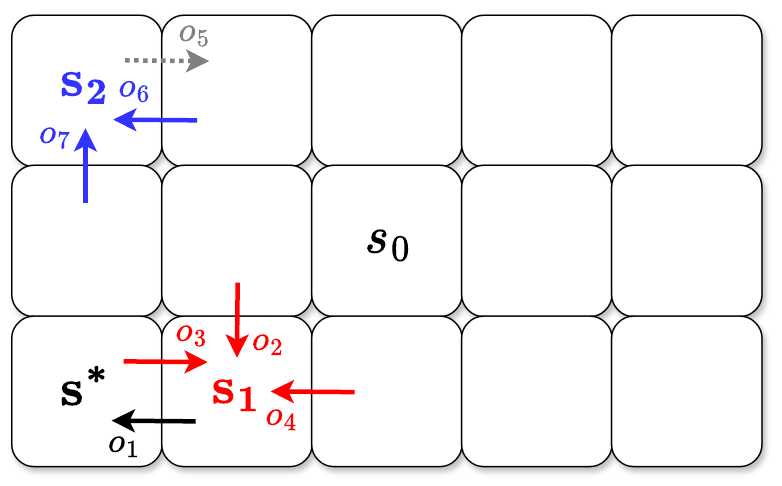}
    \caption{Example of landmark heuristic for a goal recognition task. States $s_1$ and $s_2$ and their corresponding disjunctive landmarks highlighted in red and blue, respectively.}
\label{fig:ex3}
\end{figure}

\begin{example}
Consider the example in Figure~\ref{fig:ex3} in which the goal recognition task has an initial state $\initialstate$, a goal hypothesis $\goalstate$, and two observations~$\observations=\tuple{\obs{o}_1,\obs{o}_5}$.
The observations are such that $\obs{o}_1$ is an actual step the observed plan from $\initialstate$ to $\goalstate$, and $\obs{o}_5$ is a noisy observation.
Suppose that the landmark-extraction technique generates, for each observation~$\obs{o}$ in $\observations$, a set of disjunctive action landmarks that are the operators that produce the cell from which the operator~$o$ corresponding to the observation~$\obs{o}$ is applied. 
Under these conditions both cells~$s_1$ (in red) and~$s_2$ (in blue) must be produced as landmarks from $\obs{o}_1$ and $\obs{o}_5$ respectively. 
Each of these states have a corresponding set of disjunctive landmarks, denoted by observations in the corresponding color. 
Observation~$\obs{o}_1$ generates the constraint $\Y{o_2} + \Y{o_3} + \Y{o_4} \geq [\Yobs{o_1}>0]$, and observation~$\obs{o}_2$ generates the constraint $\Y{o_6} + \Y{o_7} \geq [\Yobs{o_5} > 0]$. 
Suppose we want to compute~\holmcsg{} for~$s_0$ in Figure~\ref{fig:ex3} and that the \emph{unreliability rating}~$\unreliability$ of the sensor is equal to~$0.5$. 
Therefore, $\textup{IP/LP}$ of~\holmcsg{} only needs to increase the counts associated to one observation. 
Consequently, an optimal complying plan satisfying for~$\Yobs{o_1}$ costs three, and optimal complying plan satisfying for~$\Yobs{o_5}$ costs seven. %
Note that there are two optimal complying plans that satisfy~$\Yobs{o_1}$, and both also satisfy $\Y{o_2} + \Y{o_3} + \Y{o_4} \geq [\Yobs{o_1}>0]$, one uses~$o_2$ and another uses~$o_4$. 
Thus, the value of~\holmcsg{} for this example is three.
\end{example}

\begin{theorem}[\holmcsg{} is a Lower Bound] The value of \holmcsg{}~ solved as an $\textup{IP}$ is a lower bound on the cost of an optimal complying plan.
\label{theo:lm-heuristic-lower-bound}
\end{theorem}
\begin{proof}
To prove this property we must show that any optimal complying plan satisfies constraints \constraints, \constraintso, and $\lmc_\observations$, and that the solution to the objective function is a lower bound is a lower bound on the cost of an optimal complying plan. 
First, from Theorem~\ref{theo:hip} we have that the solution of an $\textup{IP}$ with constraints {\constraints} and {\constraintso} is a lower bound on the cost of an optimal complying plan. 
Second, from Theorem~\ref{theo:hipnoisy} we have that the solution of an $\textup{IP}$ is a lower bound on the cost of an optimal complying plan even if noisy observations are present. 
It remains to show that satisfying the constraints~$\lmc_\observations$ maintains the lower bound property.

Let $\pi = \tuple{o_1, \dots, o_n}$ be an optimal complying $s$-plan for the non-noisy observation in \observations{} for task \grplanningtask{} with goal condition \goalstate, and $s$ a state from \grplanningtask{}. 
From Theorem~\ref{theo:hipnoisy} we have that we can set each~$\Yobs{o}$ to $\occur_{\observations}(o)$ if $\obs{o} \in \pi$ and to zero otherwise. 
Thus, only non-noisy observations~$\Yobs{o}$ will have the value of the binary variable~$[\Yobs{o}>0]$ set to one. Without loss of generality, let $\lmc_\observations(\obs{o})$ be the set of disjunctive action landmarks for $\obs{o}$. Thus, we have,
\begin{align*}
    & \sum_{o\in L} \Y{o} \geq [\Yobs{o} > 0]                                 & \text{for all } L \in \lmc_\observations(\obs{o})
\end{align*}
Since~$L$ is a disjunctive action landmark for~$\obs{o}$. 
Then, every plan that satisfies condition~$pre(\obs{o})$ also satisfies~$L$. 
As the $s$-plan~$\pi$ satisfies~$pre(\obs{o})$, then the $s$-plan~$\pi$ also satisfies~$L$. 
Therefore, the counts included in~$\pi$ are sufficient to satisfy~$\lmc_\observations$. 
Thus, 
\begin{align}
    \hoipPi{\goalstate}(s) \leq \holmcsg{} \leq \cost(\pi) = \hostarPi{\goalstate}(s).
\end{align}
\end{proof}

Theorem~\ref{theo:lm-heuristic-lower-bound} ensures that the value of the $\holmcsg{}$ heuristic never overestimates the true cost of a plan that complies with all observations within the margin $\unreliability$ of observation error. 
This guarantee provides the resulting goal recognition approaches with two important properties. 
First, if $\holmcsg{}(s) = \infty$ for a goal hypothesis $\goalstate$, then any state compatible with this goal hypothesis is unreachable, and thus a recognition approach can safely filter this hypothesis out. 
Second, our lower bound guarantees that the difference between the heuristic we compute and the actual cost of an optimal plan ($\holmcsg{}(s) - \hoptimal_{\goalstate}(s)$) is never greater than that between the cost of the actual observation-complying plan and an optimal plan ($\hooptg(s) - \hoptimal_{\goalstate}(s)$). 
Thus, our strengthening of the heuristic never erroneously inflates the operator counts, and thus, whenever two goal hypotheses have similar heuristic values, their corresponding observation-complying plans will also be similar.%
\section{Experiments and Evaluation}
\label{sec:experiments and results}

This section outlines our empirical evaluation to assess the effectiveness of our goal recognition constraints in problems that reflect real-world recognition problems. 
We compare our novel LP-based approaches against key state-of-the-art approaches in \textit{Goal Recognition as Planning}~\cite{ramirez2009plan,ramirez2010probabilistic,pereira2020landmarks}.

\subsection{Experimental Setting}
\label{sec:experimental_setting}

We conducted extensive empirical experiments to evaluate how the new constraints impact the quality of the solution. 
Specifically, we compare the quality of the solutions of our base approach with observation-counting constraints against our improved approach with new specific constraints for goal recognition tasks. 
As we analyze the solution quality of each approach, we try to better understand how and when these new constraints improve the solution, as well as the cost and size of the LP models. 
Finally, we briefly compare our approaches with the compatible approaches available.

We ran all experiments with Ubuntu over an Intel Core i7 930 CPU ($2.80$GHz) with a $1$ GB memory limit. 
Our implementation uses Fast Downward version 19.06 \cite{helmert2006jair}, a Python prepossessing layer, and the CPLEX 12.10 LP solver.\footnote{Source-code and benchmarks are available at: \texttt{https://bit.ly/lp-goal-recognition}}

\begin{table}[h]
\centering
\fontsize{9}{10}\selectfont
\setlength\tabcolsep{2pt}
\begin{tabular}{c|c|c|cccccccccccc}
\toprule
\multicolumn{3}{c|}{} & \rotatebox[origin=c]{90}{ \textsc{blocks} } & \rotatebox[origin=c]{90}{ \textsc{depots} } & \rotatebox[origin=c]{90}{ \textsc{driverlog} } & \rotatebox[origin=c]{90}{ \textsc{dwr} } & \rotatebox[origin=c]{90}{ \textsc{ipc-grid} } & \rotatebox[origin=c]{90}{ \textsc{ferry} } & \rotatebox[origin=c]{90}{ \textsc{logistics} } & \rotatebox[origin=c]{90}{ \textsc{miconic} } & \rotatebox[origin=c]{90}{ \textsc{rovers} } & \rotatebox[origin=c]{90}{ \textsc{satellite} } & \rotatebox[origin=c]{90}{ \textsc{sokoban} } & \rotatebox[origin=c]{90}{ \textsc{zeno} }\\\midrule
\multicolumn{3}{c|}{$|\goalconditions|$} & 20.0 & 8.0 & 6.7 & 6.7 & 7.5 & 6.3 & 10.0 & 6.0 & 6.0 & 6.0 & 8.3 & 6.0\\\midrule
\multirow{10}{*}{ \rotatebox[origin=c]{90}{\textsc{Optimal}} }& \multicolumn{1}{c}{ \multirow{5}{*}{$|\observations|$} }& \multicolumn{1}{|c|}{10\%}& 1.5& 1.0& 1.8& 3.0& 1.5& 2.2& 2.0& 2.0& 1.7& 1.3& 2.3& 1.7\\
& \multicolumn{1}{c}{}& \multicolumn{1}{|c|}{30\%}& 3.7& 2.8& 4.2& 7.3& 3.6& 6.0& 5.8& 5.5& 3.7& 3.3& 6.3& 4.0\\
& \multicolumn{1}{c}{}& \multicolumn{1}{|c|}{50\%}& 5.2& 4.7& 6.5& 12.0& 5.4& 9.7& 9.3& 8.5& 5.7& 5.7& 10.2& 6.2\\
& \multicolumn{1}{c}{}& \multicolumn{1}{|c|}{70\%}& 7.7& 6.7& 9.2& 17.0& 7.5& 13.5& 13.2& 11.8& 8.0& 8.0& 14.7& 8.8\\
& \multicolumn{1}{c}{}& \multicolumn{1}{|c|}{100\%}& 10.3& 9.3& 12.3& 23.3& 10.2& 18.8& 18.0& 16.3& 10.7& 10.5& 20.0& 12.2\\
\cline{2-15}
& \multicolumn{1}{c}{ \multirow{5}{*}{$|\grsolution|$} }& \multicolumn{1}{|c|}{10\%} & 8.3 & 3.7 & 1.3 & 3.8 & 2.8 & 2.8 & 3.5 & 2.3 & 2.3 & 3.5 & 2.0 & 2.7\\
& \multicolumn{1}{c}{}& \multicolumn{1}{|c|}{30\%} & 2.0 & 1.7 & 2.0 & 2.0 & 1.2 & 1.3 & 1.0 & 1.0 & 1.2 & 3.5 & 1.5 & 1.3\\
& \multicolumn{1}{c}{}& \multicolumn{1}{|c|}{50\%} & 1.2 & 1.3 & 1.2 & 1.3 & 1.1 & 1.2 & 1.0 & 1.0 & 1.2 & 1.7 & 1.3 & 1.2\\
& \multicolumn{1}{c}{}& \multicolumn{1}{|c|}{70\%} & 1.2 & 1.2 & 1.0 & 1.2 & 1.1 & 1.2 & 1.2 & 1.0 & 1.0 & 1.5 & 1.0 & 1.0\\
& \multicolumn{1}{c}{}& \multicolumn{1}{|c|}{100\%} & 1.2 & 1.0 & 1.0 & 1.0 & 1.0 & 1.0 & 1.0 & 1.0 & 1.0 & 1.3 & 1.0 & 1.0\\
\midrule
\multirow{10}{*}{ \rotatebox[origin=c]{90}{\textsc{Suboptimal}} }& \multicolumn{1}{c}{ \multirow{5}{*}{$|\observations|$} }& \multicolumn{1}{|c|}{10\%}& 1.7& 1.8& 2.2& 3.5& 1.9& 3.0& 2.7& 3.0& 1.8& 2.0& 3.3& 2.0\\
& \multicolumn{1}{c}{}& \multicolumn{1}{|c|}{30\%}& 4.3& 4.5& 5.7& 9.7& 5.1& 8.5& 7.3& 7.8& 4.3& 4.3& 8.7& 5.5\\
& \multicolumn{1}{c}{}& \multicolumn{1}{|c|}{50\%}& 6.7& 6.7& 9.0& 15.3& 8.1& 13.7& 11.7& 12.5& 7.0& 6.7& 13.5& 8.3\\
& \multicolumn{1}{c}{}& \multicolumn{1}{|c|}{70\%}& 9.7& 9.7& 12.7& 21.3& 11.4& 19.2& 16.3& 17.5& 9.8& 9.2& 19.0& 12.0\\
& \multicolumn{1}{c}{}& \multicolumn{1}{|c|}{100\%}& 13.3& 13.3& 17.3& 30.0& 15.8& 26.8& 22.7& 24.3& 13.3& 12.5& 26.7& 16.5\\
\cline{2-15}
& \multicolumn{1}{c}{ \multirow{5}{*}{$|\grsolution|$} }& \multicolumn{1}{|c|}{10\%} & 7.3 & 1.3 & 2.2 & 3.8 & 1.6 & 3.0 & 1.5 & 1.8 & 2.5 & 3.8 & 1.8 & 2.2\\
& \multicolumn{1}{c}{}& \multicolumn{1}{|c|}{30\%} & 2.0 & 1.5 & 1.3 & 1.2 & 1.4 & 1.2 & 1.2 & 1.0 & 1.3 & 1.7 & 1.3 & 1.2\\
& \multicolumn{1}{c}{}& \multicolumn{1}{|c|}{50\%} & 1.2 & 1.0 & 1.0 & 1.2 & 2.1 & 1.2 & 1.0 & 1.0 & 1.0 & 1.7 & 1.7 & 1.0\\
& \multicolumn{1}{c}{}& \multicolumn{1}{|c|}{70\%} & 1.5 & 1.0 & 1.2 & 1.2 & 2.1 & 1.2 & 1.0 & 1.0 & 1.0 & 1.3 & 1.5 & 1.0\\
& \multicolumn{1}{c}{}& \multicolumn{1}{|c|}{100\%} & 1.2 & 1.0 & 1.0 & 1.0 & 2.0 & 1.2 & 1.0 & 1.0 & 1.0 & 1.3 & 1.5 & 1.0\\
\bottomrule
\end{tabular}\\ %
\caption{Key properties of each experimental domain.}
\label{tab:benchmark-domains} %
\end{table}

For this analysis, we create a new benchmark by adapting the benchmark introduced by \citeauthor{pereira2017landmark} \shortcite{pereira2017landmark}. 
We evaluate the approaches primarily by using the \emph{agreement ratio} evaluation metric from~\citeauthor{ramirez2009plan}~\shortcite{ramirez2009plan}. 
Recall that Ramirez and Geffner define the \emph{agreement ratio} as the intersection over union~$|\grsolution \cap \goalconditions|/|\grsolution \cup \goalconditions|$ of the reference solution set~$\grsolution$ against the solution~$\goalconditions$ provided by the approach. 
We use \emph{agreement ratio} because it penalizes two key potential flaws in solutions to goal recognition problems. 
First, approaches that return many goals that are not part of the reference solution set, i.e., the approach returns false negatives. 
Second, approaches that return few goals that are part of the reference solution set, i.e., the approach fails to return goals compatible with the evidence. 
Thus, using the \emph{agreement ratio} we have a single metric that accounts for the two most important aspects when evaluating approaches for solving goal recognition tasks.

For each domain, we create three base planning tasks~$\planningtask_{\textup{P}}$ each (except for \textsc{IPC-Grid}, in which we create four) with four reference goals conditions each. 
For each pair (planning tasks~$\planningtask_{\textup{P}}$, reference goal condition) we compute a plan from which we extract the sequence of observations. 
We compute optimal and suboptimal plans for each pair creating two different data sets. 
We compute suboptimal plans using weighted A$^*$ with $w = 2$~\cite{pohl1970heuristic} to emulate bounded rationality.

Following previous work, we build the benchmark data set with five different levels of observability: 10\%, 30\%, 50\%, 70\% and 100\%. 
We only generate one sequence of observations for 100\% of observability, and three different random observation sequences extracted from the same plan for other observability levels, yielding 208 goal recognition tasks in total for \textsc{IPC-Grid} and 156 for each one of the other domains in each data set (optimal and suboptimal). 
For each data set, we also create an additional corresponding noisy data set by adding $\lceil|\observations| * 0.2\rceil$ randomly generated observations in each sequence of observations---i.e., the fault rate of the sensor is 20\%. 
For each goal recognition task we add at least five randomly generated candidate goal conditions. 
In total, we have $8,288$ goal recognition tasks divided into four data sets.

We compute the reference solution set~$\grsolution$ for each goal recognition task for optimal and suboptimal data sets. 
Thus, for each pair of base planning task and goal candidate, we compute an optimal plan, and one optimal complying plan for each sequence of observations of an optimal data set, and one bounded suboptimal plan and bounded suboptimal complying plan. 
This a computationally intensive problem, but is performed only once for the creation of the benchmark dataset. 
The noisy data sets have the same solution as their corresponding optimal and suboptimal counterparts. 
For the optimal data set we use A$^*$, whereas for the suboptimal data set we use weighted A$^*$ with $w = \frac{\cost(\plan)}{\hstarPi{\rgoal}(\initialstate)}$ (following Definition~\ref{def:gr_solutionset}). 
To ensure replicability and facilitate further research on goal recognition, this benchmark is available publicly.\footnote{At \url{https://github.com/luisaras/goal-plan-recognition-dataset}}
Table~\ref{tab:benchmark-domains} summarizes the information about the data sets. 
The domains we use are \textsc{Blocks}, \textsc{Depots}, \textsc{Driverlog}, \textsc{DWR}, \textsc{IPC-Grid}, \textsc{Ferry}, \textsc{Logistics}, \textsc{Miconic}, \textsc{Rovers}, \textsc{Satellite}, \textsc{Sokoban} and \textsc{Zeno}. 
For each domain row, $|\goalconditions|$ represents the average number of candidate goals.
Columns $|\observations|$ and $|\grsolution|$ show the average size of the observations and the reference solution set, respectively. 
The average size of the plan with 100\% of observability indicates the size of the plan computed for the reference goal.
As expected, the average sizes $|\observations|$ and $|\grsolution|$ are larger for the suboptimal data set than for the optimal data set. %

\subsection{Evaluation and Comparison of the Novel LP-Based Approaches}
\label{sec:evaluation}

In what follows, we evaluate our extensions to the LP heuristic with \textit{Landmark Constraints} ({\holmcsg{}}) and compare it to its ``basic'' version (\holmc{}). 
Tables~\ref{tab:lm1} and~\ref{tab:lm2} summarize these results. 
Tables~\ref{tab:lmc-optimal}--\ref{tab:lmc-optimal-noisy} summarize the results over all domains for ``simplified'' problems, that is problems that do not include suboptimal or noisy observations. 
More specifically: Table~\ref{tab:lmc-optimal} shows results for the data set with noise-free observations from optimal plans; Table~\ref{tab:lmc-sub-optimal} shows results for the data set with noise-free observations from suboptimal plans, and Table~\ref{tab:lmc-optimal-noisy} shows results from the data set with noisy observations from optimal plans. 
By contrast, Table~\ref{tab:lm2} shows results for the hardest version of the goal recognition problems, which consist of observations generated from suboptimal plans with added noise.  
For both noisy data sets, we use a noise filter (see Section~\ref{sec:addressing_noisy_observations}) of $\epsilon = 0.2$. 
Each table compares the heuristics across five metrics: 
\textbf{Agr} is the average \emph{agreement ratio} for that set of problems; 
\textbf{\ho} is the average value of that version of the heuristic for the \emph{reference goal}; 
\textbf{Rows} denote the average number of constraints (i.e., rows) in the LP responsible for computing the heuristic; 
the \textbf{total time} is the time, in seconds, including both the Fast Downward and Python calls; and 
\textbf{LP time} is the time to solve the LP in the Fast Downward after all the pre-processing. 

\begin{table*}
    \begin{center}
    \begin{subtable}[b]{.6\textwidth}
        \begin{center}
        \fontsize{8.}{9.}\selectfont
        \setlength\tabcolsep{1pt}
\begin{tabular}{cc|ccccc|ccccc}
\toprule
\multicolumn{2}{c}{} %
& \multicolumn{5}{c|}{\holmc{}}
& \multicolumn{5}{c}{\holmcsg{}}\\
\cmidrule(lr){3-7} \cmidrule(lr){8-12}
\multicolumn{2}{c}{}& & & & \multicolumn{2}{c}{Time}& & & & \multicolumn{2}{c}{Time}\\
\cmidrule(lr){6-7} \cmidrule(lr){11-12}
\# & \textbf{\%}
& \textbf{Agr}  & \textbf{$h^\Omega$}  & \textbf{Rows}  & \textbf{Total}  & \textbf{LP} 
& \textbf{Agr}  & \textbf{$h^\Omega$}  & \textbf{Rows}  & \textbf{Total}  & \textbf{LP} 

\\ 
\hline %

\multicolumn{2}{c|}{10\%} 
& 0.69 & 11.2 & 14.9 & 0.62 & 0.01 	 

& \textbf{0.71} & 11.4 & 19.4 & 0.62 & 0.01 	 
 \\
\multicolumn{2}{c|}{30\%} 
& 0.67 & 11.8 & 17.6 & 0.62 & 0.01 	 

& \textbf{0.73} & 12.1 & 29.3 & 0.63 & 0.01 	 
 \\
\multicolumn{2}{c|}{50\%} 
& 0.75 & 12.3 & 20.2 & 0.62 & 0.01 	 

& \textbf{0.78} & 12.8 & 39.0 & 0.63 & 0.02 	 
 \\
\multicolumn{2}{c|}{70\%} 
& 0.84 & 13.1 & 23.1 & 0.62 & 0.01 	 

& \textbf{0.86} & 13.5 & 49.7 & 0.63 & 0.02 	 
 \\
\multicolumn{2}{c|}{100\%} 
& \textbf{0.91} & 14.3 & 26.6 & 0.62 & 0.01 	 

& \textbf{0.91} & 14.3 & 62.7 & 0.63 & 0.02 	 
 \\\hline
\multicolumn{2}{c|}{AVG} & 0.77 & 12.5 & 20.5 & 0.62 & 0.01& \textbf{0.80} & 12.8 & 40.0 & 0.63 & 0.02 %
\\ \bottomrule
\end{tabular}

         \caption{Results for optimal noise-free observations.}
        \label{tab:lmc-optimal}
        \end{center}
    \end{subtable}
    \quad
    \begin{subtable}[b]{.6\textwidth}
        \begin{center}
        \fontsize{8.}{9.}\selectfont
        \setlength\tabcolsep{1pt}
\begin{tabular}{cc|ccccc|ccccc}
\toprule
\multicolumn{2}{c}{} %
& \multicolumn{5}{c|}{\holmc{}}
& \multicolumn{5}{c}{\holmcsg{}}\\
\cmidrule(lr){3-7} \cmidrule(lr){8-12}
\multicolumn{2}{c}{}& & & & \multicolumn{2}{c}{Time}& & & & \multicolumn{2}{c}{Time}\\
\cmidrule(lr){6-7} \cmidrule(lr){11-12}
\# & \textbf{\%}
& \textbf{Agr}  & \textbf{$h^\Omega$}  & \textbf{Rows}  & \textbf{Total}  & \textbf{LP} 
& \textbf{Agr}  & \textbf{$h^\Omega$}  & \textbf{Rows}  & \textbf{Total}  & \textbf{LP} 

\\ 
\hline %

\multicolumn{2}{c|}{10\%} 
& 0.6 & 11.6 & 15.4 & 0.62 & 0.01 	 

& \textbf{0.66} & 11.7 & 21.6 & 0.62 & 0.01 	 
 \\
\multicolumn{2}{c|}{30\%} 
& 0.67 & 12.8 & 18.9 & 0.62 & 0.01 	 

& \textbf{0.73} & 13.2 & 34.7 & 0.63 & 0.01 	 
 \\
\multicolumn{2}{c|}{50\%} 
& 0.73 & 14.3 & 22.0 & 0.62 & 0.01 	 

& \textbf{0.76} & 14.7 & 45.9 & 0.63 & 0.02 	 
 \\
\multicolumn{2}{c|}{70\%} 
& 0.82 & 16.3 & 25.1 & 0.62 & 0.01 	 

& \textbf{0.83} & 16.6 & 58.0 & 0.63 & 0.02 	 
 \\
\multicolumn{2}{c|}{100\%} 
& \textbf{0.88} & 19.3 & 28.9 & 0.62 & 0.01 	 

& \textbf{0.88} & 19.3 & 72.6 & 0.64 & 0.02 	 
 \\\hline
\multicolumn{2}{c|}{AVG} & 0.74 & 14.9 & 22.1 & 0.62 & 0.01& \textbf{0.77} & 15.1 & 46.6 & 0.63 & 0.02 %
\\ \bottomrule
\end{tabular}

         \end{center}
        \caption{Results for suboptimal noise-free observations.}
        \label{tab:lmc-sub-optimal}
    \end{subtable}
    \quad
    \begin{subtable}[b]{.6\textwidth}
        \begin{center}
        \fontsize{8.}{9.}\selectfont
        \setlength\tabcolsep{1pt}
\begin{tabular}{cc|ccccc|ccccc}
\toprule
\multicolumn{2}{c}{} %
& \multicolumn{5}{c|}{\holmc{}}
& \multicolumn{5}{c}{\holmcsg{}}\\
\cmidrule(lr){3-7} \cmidrule(lr){8-12}
\multicolumn{2}{c}{}& & & & \multicolumn{2}{c}{Time}& & & & \multicolumn{2}{c}{Time}\\
\cmidrule(lr){6-7} \cmidrule(lr){11-12}
\# & \textbf{\%}
& \textbf{Agr}  & \textbf{$h^\Omega$}  & \textbf{Rows}  & \textbf{Total}  & \textbf{LP} 
& \textbf{Agr}  & \textbf{$h^\Omega$}  & \textbf{Rows}  & \textbf{Total}  & \textbf{LP} 

\\ 
\hline %

\multicolumn{2}{c|}{10\%} 
& 0.48 & 11.2 & 15.5 & 0.62 & 0.01 	 

& \textbf{0.49} & 11.3 & 22.7 & 0.62 & 0.01 	 
 \\
\multicolumn{2}{c|}{30\%} 
& 0.51 & 11.6 & 18.5 & 0.62 & 0.01 	 

& \textbf{0.55} & 11.9 & 33.8 & 0.63 & 0.01 	 
 \\
\multicolumn{2}{c|}{50\%} 
& 0.62 & 12.0 & 21.5 & 0.62 & 0.01 	 

& \textbf{0.68} & 12.4 & 45.0 & 0.63 & 0.02 	 
 \\
\multicolumn{2}{c|}{70\%} 
& 0.79 & 12.5 & 24.8 & 0.62 & 0.01 	 

& \textbf{0.81} & 12.9 & 57.6 & 0.63 & 0.02 	 
 \\
\multicolumn{2}{c|}{100\%} 
& 0.87 & 13.6 & 28.9 & 0.62 & 0.01 	 

& \textbf{0.88} & 13.6 & 74.0 & 0.64 & 0.03 	 
 \\\hline
\multicolumn{2}{c|}{AVG} & 0.66 & 12.2 & 21.8 & 0.62 & 0.01& \textbf{0.68} & 12.4 & 46.6 & 0.63 & 0.02 %
\\ \bottomrule
\end{tabular}

         \caption{Results for noisy observations (optimal plans).}
        \label{tab:lmc-optimal-noisy}
        \end{center}
    \end{subtable}
    \caption{Aggregated average results using the \textit{Landmark} constraints.}
    \label{tab:lm1}
    \end{center}
\end{table*}

\begin{table*}
    \begin{center}
    \fontsize{8.}{9.}\selectfont
    \setlength\tabcolsep{1pt}
\begin{tabular}{cc|ccccc|ccccc}
\toprule
\multicolumn{2}{c}{} %
& \multicolumn{5}{c|}{\holmc{}}
& \multicolumn{5}{c}{\holmcsg{}}\\
\cmidrule(lr){3-7} \cmidrule(lr){8-12}
\multicolumn{2}{c}{}& & & & \multicolumn{2}{c}{Time}& & & & \multicolumn{2}{c}{Time}\\
\cmidrule(lr){6-7} \cmidrule(lr){11-12}
\# & \textbf{\%}
& \textbf{Agr}  & \textbf{$h^\Omega$}  & \textbf{Rows}  & \textbf{Total}  & \textbf{LP} 
& \textbf{Agr}  & \textbf{$h^\Omega$}  & \textbf{Rows}  & \textbf{Total}  & \textbf{LP} 

\\ 
\hline %

\multirow{5}{*}{\rotatebox[origin=c]{90}{\textsc{blocks}}} 
	 & 10

& 0.35 & 6.3 & 10.2 & 1.46 & 0.02 	 

& \textbf{0.37} & 6.4 & 14.3 & 1.46 & 0.02 	 

	\\ & 30

& \textbf{0.38} & 7.3 & 12.5 & 1.46 & 0.02 	 

& 0.35 & 7.4 & 19.2 & 1.47 & 0.03 	 

	\\ & 50

& \textbf{0.45} & 7.5 & 14.7 & 1.46 & 0.02 	 

& 0.42 & 7.7 & 23.6 & 1.47 & 0.03 	 

	\\ & 70

& \textbf{0.44} & 9.4 & 17.5 & 1.46 & 0.02 	 

& 0.43 & 9.6 & 30.4 & 1.48 & 0.03 	 

	\\ & 100

& \textbf{0.56} & 11.1 & 20.7 & 1.46 & 0.02 	 

& 0.54 & 11.1 & 38.3 & 1.48 & 0.04 	 
 \\ \hline
\multirow{5}{*}{\rotatebox[origin=c]{90}{\textsc{depots}}} 
	 & 10

& 0.30 & 6.8 & 10.2 & 0.62 & 0.01 	 

& \textbf{0.37} & 7.1 & 14.5 & 0.62 & 0.01 	 

	\\ & 30

& 0.32 & 7.1 & 12.8 & 0.62 & 0.01 	 

& \textbf{0.40} & 7.9 & 22.9 & 0.62 & 0.01 	 

	\\ & 50

& 0.38 & 8.7 & 14.9 & 0.62 & 0.01 	 

& \textbf{0.65} & 9.6 & 30.6 & 0.63 & 0.01 	 

	\\ & 70

& 0.66 & 10.2 & 17.4 & 0.62 & 0.01 	 

& \textbf{0.77} & 10.6 & 38.3 & 0.63 & 0.02 	 

	\\ & 100

& 0.9 & 12.8 & 20.0 & 0.62 & 0.01 	 

& \textbf{0.93} & 12.8 & 46.9 & 0.63 & 0.02 	 
 \\ \hline
\multirow{5}{*}{\rotatebox[origin=c]{90}{\textsc{driverlog}}} 
	 & 10

& 0.38 & 10.1 & 14.4 & 0.47 & 0.01 	 

& \textbf{0.40} & 10.2 & 22.7 & 0.47 & 0.01 	 

	\\ & 30

& 0.44 & 10.6 & 18.1 & 0.47 & 0.01 	 

& \textbf{0.55} & 11.4 & 37.0 & 0.47 & 0.01 	 

	\\ & 50

& 0.51 & 11.9 & 21.1 & 0.47 & 0.01 	 

& \textbf{0.60} & 12.4 & 48.4 & 0.48 & 0.01 	 

	\\ & 70

& \textbf{0.70} & 13.2 & 24.6 & 0.47 & 0.01 	 

& 0.69 & 13.6 & 61.2 & 0.48 & 0.01 	 

	\\ & 100

& 0.65 & 16.2 & 28.9 & 0.47 & 0.01 	 

& \textbf{0.68} & 16.2 & 79.2 & 0.48 & 0.02 	 
 \\ \hline
\multirow{5}{*}{\rotatebox[origin=c]{90}{\textsc{dwr}}} 
	 & 10

& \textbf{0.44} & 11.3 & 16.5 & 0.54 & 0.01 	 

& 0.41 & 11.6 & 28.2 & 0.55 & 0.01 	 

	\\ & 30

& \textbf{0.37} & 14.1 & 21.9 & 0.54 & 0.01 	 

& 0.35 & 15.1 & 52.2 & 0.55 & 0.02 	 

	\\ & 50

& 0.38 & 16.9 & 26.2 & 0.54 & 0.01 	 

& \textbf{0.45} & 17.8 & 74.4 & 0.56 & 0.02 	 

	\\ & 70

& 0.54 & 21.0 & 29.9 & 0.54 & 0.01 	 

& \textbf{0.59} & 21.9 & 91.7 & 0.56 & 0.03 	 

	\\ & 100

& 0.73 & 27.0 & 35.3 & 0.54 & 0.01 	 

& \textbf{0.78} & 27.0 & 120.3 & 0.57 & 0.03 	 
 \\ \hline
\multirow{5}{*}{\rotatebox[origin=c]{90}{\textsc{ipc-grid}}} 
	 & 10

& 0.68 & 11.3 & 15.1 & 0.58 & 0.01 	 

& \textbf{0.77} & 11.3 & 28.8 & 0.58 & 0.01 	 

	\\ & 30

& \textbf{0.77} & 12.1 & 18.5 & 0.58 & 0.01 	 

& \textbf{0.77} & 12.2 & 50.0 & 0.59 & 0.01 	 

	\\ & 50

& \textbf{0.88} & 12.7 & 21.5 & 0.58 & 0.01 	 

& \textbf{0.88} & 12.8 & 68.4 & 0.59 & 0.02 	 

	\\ & 70

& \textbf{0.90} & 14.3 & 24.8 & 0.58 & 0.01 	 

& \textbf{0.90} & 14.4 & 88.5 & 0.59 & 0.02 	 

	\\ & 100

& \textbf{0.94} & 16.6 & 28.1 & 0.58 & 0.01 	 

& \textbf{0.94} & 16.6 & 110.2 & 0.59 & 0.02 	 
 \\ \hline
\multirow{5}{*}{\rotatebox[origin=c]{90}{\textsc{ferry}}} 
	 & 10

& 0.37 & 14.0 & 19.4 & 0.40 & 0.01 	 

& \textbf{0.38} & 14.0 & 25.8 & 0.40 & 0.01 	 

	\\ & 30

& \textbf{0.51} & 15.6 & 25.1 & 0.40 & 0.01 	 

& \textbf{0.51} & 15.6 & 40.0 & 0.40 & 0.01 	 

	\\ & 50

& \textbf{0.81} & 18.5 & 29.7 & 0.40 & 0.01 	 

& \textbf{0.81} & 18.5 & 52.6 & 0.40 & 0.01 	 

	\\ & 70

& \textbf{0.87} & 21.4 & 34.0 & 0.40 & 0.01 	 

& \textbf{0.87} & 21.4 & 64.6 & 0.40 & 0.01 	 

	\\ & 100

& \textbf{0.94} & 26.4 & 39.4 & 0.40 & 0.01 	 

& \textbf{0.94} & 26.4 & 80.4 & 0.40 & 0.01 	 
 \\ \hline
\multirow{5}{*}{\rotatebox[origin=c]{90}{\textsc{logistics}}} 
	 & 10

& 0.59 & 17.0 & 21.8 & 0.69 & 0.01 	 

& \textbf{0.62} & 17.0 & 33.8 & 0.69 & 0.01 	 

	\\ & 30

& \textbf{0.84} & 17.4 & 26.5 & 0.69 & 0.01 	 

& \textbf{0.84} & 17.4 & 56.2 & 0.70 & 0.02 	 

	\\ & 50

& \textbf{0.95} & 18.1 & 30.9 & 0.69 & 0.01 	 

& \textbf{0.95} & 18.1 & 76.4 & 0.70 & 0.02 	 

	\\ & 70

& 0.94 & 19.8 & 35.4 & 0.69 & 0.01 	 

& \textbf{0.96} & 19.9 & 96.9 & 0.71 & 0.02 	 

	\\ & 100

& \textbf{1.0} & 22.2 & 41.1 & 0.69 & 0.01 	 

& \textbf{1.0} & 22.3 & 122.8 & 0.71 & 0.03 	 
 \\ \hline
\multirow{5}{*}{\rotatebox[origin=c]{90}{\textsc{miconic}}} 
	 & 10

& 0.46 & 15.9 & 21.8 & 0.43 & 0.01 	 

& \textbf{0.54} & 15.9 & 26.5 & 0.43 & 0.01 	 

	\\ & 30

& 0.72 & 16.6 & 27.0 & 0.43 & 0.01 	 

& \textbf{0.80} & 16.8 & 38.3 & 0.43 & 0.01 	 

	\\ & 50

& 0.86 & 17.6 & 32.0 & 0.43 & 0.01 	 

& \textbf{0.93} & 17.6 & 48.9 & 0.43 & 0.01 	 

	\\ & 70

& 0.90 & 19.4 & 36.8 & 0.43 & 0.01 	 

& \textbf{0.93} & 19.5 & 60.0 & 0.44 & 0.02 	 

	\\ & 100

& \textbf{1.0} & 23.0 & 43.0 & 0.43 & 0.01 	 

& \textbf{1.0} & 23.0 & 74.3 & 0.44 & 0.02 	 
 \\ \hline
\multirow{5}{*}{\rotatebox[origin=c]{90}{\textsc{rovers}}} 
	 & 10

& \textbf{0.54} & 9.7 & 14.1 & 0.47 & 0.01 	 

& 0.52 & 9.9 & 17.8 & 0.47 & 0.01 	 

	\\ & 30

& 0.65 & 10.2 & 16.9 & 0.47 & 0.01 	 

& \textbf{0.66} & 10.5 & 24.4 & 0.47 & 0.01 	 

	\\ & 50

& 0.85 & 10.6 & 19.7 & 0.47 & 0.01 	 

& \textbf{0.86} & 11.1 & 31.4 & 0.48 & 0.01 	 

	\\ & 70

& \textbf{0.96} & 11.9 & 22.1 & 0.47 & 0.01 	 

& 0.94 & 12.3 & 36.9 & 0.47 & 0.01 	 

	\\ & 100

& \textbf{1.0} & 13.0 & 25.7 & 0.47 & 0.01 	 

& \textbf{1.0} & 13.0 & 46.4 & 0.48 & 0.01 	 
 \\ \hline
\multirow{5}{*}{\rotatebox[origin=c]{90}{\textsc{satellite}}} 
	 & 10

& 0.51 & 10.1 & 14.7 & 0.40 & 0.01 	 

& \textbf{0.52} & 10.2 & 19.9 & 0.41 & 0.01 	 

	\\ & 30

& 0.55 & 10.3 & 17.1 & 0.40 & 0.01 	 

& \textbf{0.59} & 10.4 & 26.0 & 0.41 & 0.01 	 

	\\ & 50

& 0.71 & 10.7 & 20.3 & 0.41 & 0.01 	 

& \textbf{0.77} & 10.9 & 35.6 & 0.41 & 0.01 	 

	\\ & 70

& 0.87 & 11.1 & 22.9 & 0.41 & 0.01 	 

& \textbf{0.89} & 11.3 & 41.6 & 0.41 & 0.01 	 

	\\ & 100

& \textbf{0.97} & 12.3 & 26.1 & 0.41 & 0.01 	 

& \textbf{0.97} & 12.3 & 51.2 & 0.41 & 0.01 	 
 \\ \hline
\multirow{5}{*}{\rotatebox[origin=c]{90}{\textsc{sokoban}}} 
	 & 10

& 0.27 & 15.1 & 20.4 & 0.9 & 0.02 	 

& \textbf{0.30} & 15.2 & 41.6 & 0.91 & 0.03 	 

	\\ & 30

& 0.32 & 16.8 & 26.4 & 0.90 & 0.02 	 

& \textbf{0.40} & 17.9 & 81.0 & 0.93 & 0.06 	 

	\\ & 50

& 0.43 & 19.0 & 31.6 & 0.90 & 0.02 	 

& \textbf{0.50} & 20.4 & 120.3 & 0.95 & 0.08 	 

	\\ & 70

& \textbf{0.57} & 21.4 & 37.4 & 0.90 & 0.02 	 

& 0.54 & 22.5 & 154.6 & 0.98 & 0.10 	 

	\\ & 100

& 0.70 & 25.7 & 44.2 & 0.90 & 0.02 	 

& \textbf{0.73} & 25.7 & 207.1 & 1.0 & 0.12 	 
 \\ \hline
\multirow{5}{*}{\rotatebox[origin=c]{90}{\textsc{zeno}}} 
	 & 10

& 0.43 & 10.0 & 14.8 & 0.51 & 0.01 	 

& \textbf{0.50} & 10.1 & 20.5 & 0.51 & 0.01 	 

	\\ & 30

& 0.61 & 10.8 & 18.4 & 0.51 & 0.01 	 

& \textbf{0.70} & 11.1 & 31.3 & 0.51 & 0.01 	 

	\\ & 50

& 0.82 & 11.4 & 21.5 & 0.51 & 0.01 	 

& \textbf{0.86} & 11.8 & 41.2 & 0.52 & 0.02 	 

	\\ & 70

& \textbf{0.93} & 13.2 & 24.9 & 0.51 & 0.01 	 

& 0.92 & 13.5 & 51.9 & 0.52 & 0.02 	 

	\\ & 100

& 0.92 & 15.5 & 29.5 & 0.51 & 0.01 	 

& \textbf{0.94} & 15.5 & 65.3 & 0.52 & 0.02 	 
 \\ \hline
\multicolumn{2}{c|}{AVG} & 0.66 & 14.3 & 23.8 & 0.62 & 0.01& \textbf{0.69} & 14.6 & 54.7 & 0.63 & 0.02 %
\\ \bottomrule
\end{tabular}

     \caption{Results and comparison per domain for the \textit{Landmark} constraints on suboptimal and noisy observations.}
    \label{tab:lm2}
    \end{center}
\end{table*}

In the simplest variation of the data set from Table~\ref{tab:lmc-optimal}, the {\holmcsg{}} heuristic dominates the basic \holmc{} heuristic, at the cost of roughly double the number of constraints in the LP. 
Nevertheless, these additional constraints do not lead to a commensurate increasing in computational time in the recognition time. 
We observe a similar behavior for the {\holmcsg{}} in the suboptimal data set from Table~\ref{tab:lmc-sub-optimal}. 
Interestingly, the improvement in agreement for suboptimal plans seems to be more pronounced at lower levels of observability. 
This improvement stems primarily from the inclusion of extra landmarks constraints derived from the observations. 
We can clearly see this increase in the total number of rows of the LP, which increases as a function of the observability level (i.e., the more observations, the more constraints). 
When these observations are noise-free, they compensate for the lower probability with which partial observations might intersect with the landmarks of the planning problems corresponding to each goal hypothesis. 
Indeed, this points to our novel approach overcoming a known limitation of the landmark-based recognition approach from~\citet{pereira2020landmarks}. 
For the optimal, but noisy data set, not only does the {\holmcsg{}} heuristic dominate \holmc{}, it also consistently outperforms \holmc{} in agreement. 
While the computational cost of the resulting LP seems to increase more substantially, this additional cost is still negligible compared to the total cost of the recognition process. 

Finally, we provide an analysis of the results for noisy observations from suboptimal plans in Table~\ref{tab:lm2} segmented by the different problem domains paints a more complex picture of our results.
While the value of the {\holmcsg{}} heuristic still dominates \holmc{} across the board, its has a less pronounced impact on agreement ratio, depending on the domain. 
This is especially true for the blocks world domain. 
In summary for all benchmarks, the additional constraints either improve the agreement ratio overall, or tie with the base \holmc{} heuristic. 
These improvements come at a doubling of the number of constraints and a very small increase in the cost of solving the corresponding LP in most domains. 
Nevertheless, problems in the Sokoban domain show a pronounced increase in the size of the LP, this is likely due to the number of disjunctive landmarks induced by the grid setting from Sokoban.
\frm{When we break the results down into the different domains for the other benchmarks, do we see similar behavior. If not, revise the next sentence.}
We note that the behavior we describe for Table~\ref{tab:lm2} is similar across the domains in the data set we summarize Table~\ref{tab:lm1}.

\begin{figure}[h]
    \centering
    \rotatebox[origin=c]{90}{\holmcsg{}}
    \hspace{-1em}
\setlength\tabcolsep{1pt}
\begin{tabular}{ccccc}
10\% & 30\% & 50\% & 70\% & 100\%\\
\includegraphics[width=0.18\linewidth]{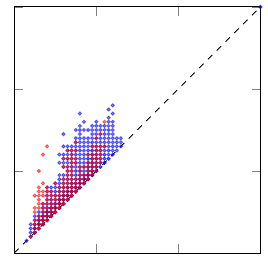}
& \includegraphics[width=0.18\linewidth]{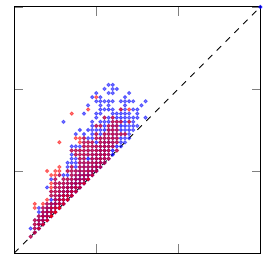}
& \includegraphics[width=0.18\linewidth]{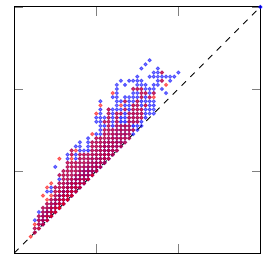}
& \includegraphics[width=0.18\linewidth]{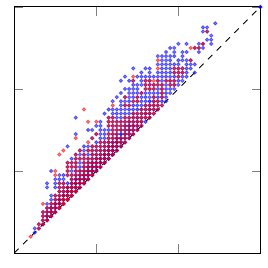}
& \includegraphics[width=0.18\linewidth]{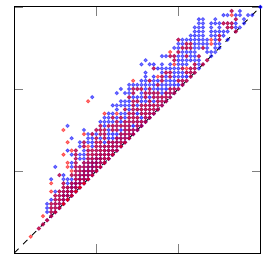}
\\6855/7545/0 & 9545/4855/0 & 10139/4261/0 & 9423/4977/0 & 4933/4667/0
\end{tabular}
    \\\holmc
    \caption{Scatter plot comparing the heuristic values of \holmc{} with \holmcsg{}, for observability levels 10\%, 30\%, 50\%, 70\% and 100\%. We take into consideration all goal hypotheses from all instances, and all four data sets. Red indicates the value for the reference goal, blue indicates the value for all other goal hypotheses.}
    \label{fig:lmc-h-value-scatter-color}
\end{figure}

We use scatter plots to better visualize the behavior of each heuristic. 
Figure~\ref{fig:lmc-h-value-scatter-color}, contain four scatter plots, one for each observability ratio (10\%, 30\%, 50\%, 70\%, 100\%) illustrating the values of the {\holmcsg{}/{\holmc} heuristics for both the reference goal (in red), and for all other goal hypotheses. 
The x-axis is the value of the basic heuristic without our modifications, and in the y-axis is the value of the modified heuristic. 
Each point in the plot is one goal candidate, for all candidates in all instances of all domains combined, and below each plot we show three numbers separated by a ``/'': number of points above the diagonal, on the diagonal, and below it. 
For all problems, the value of {\holmcsg{}} dominates {\holmc} as indicated by all points being above the diagonal. 
Importantly, most red points (correct hypotheses) concentrate toward the left of the point cloud and above the corresponding blue points (incorrect hypotheses).
This indicates that the value of {\holmcsg{}} dominates {\holmc}, which suggests that the value of {\holmcsg{}} increases more for incorrect hypotheses than for correct ones, and thus helps in differentiating goals. 
As we increase observability, the point cloud flattens towards the diagonal, indicating that the additional constraints from \holmcsg{} yield values closer to the previous {\holmc} heuristic, which is an expected result.

Finally, while the agreement ratio from the experiments in this section are lower than for the version of our framework in \cite{Santos2021} that combines the constraints for State Equation \textbf{and} the original Landmarks constraints from Definition~\ref{def:lm_constraints}, shown in the columns labeled ``S,L'' in \cite{Santos2021}. 
However, the comparable columns to our results in this paper are the ones labelled ``L'' alone, which correspond to the columns labeled `$\lmc$' in this article. 
Indeed, this paper deliberately chooses to focus on the landmarks constraints in order to explore the theoretical properties of linear programming constraints within our framework. 
As expected, adding all the sources of constraints from that paper strengthens our heuristic and thus yield a similar gain in recognition accuracy.

\subsection{Comparison with Previous Recognition Approaches}
\label{sec:previous_methods}

\begin{table}[tb]
    \begin{center}
    \begin{subtable}[t]{.45\textwidth}
        \begin{center}
        \fontsize{8.}{9.}\selectfont
        \setlength\tabcolsep{2pt}
\begin{tabular}{cc|cc|cc|cc}
\toprule
\multicolumn{2}{c}{} %
& \multicolumn{2}{c|}{\rg}
& \multicolumn{2}{c|}{\pom}
& \multicolumn{2}{c}{\holmcsg{}}\\
\cmidrule(lr){3-4} \cmidrule(lr){5-6} \cmidrule(lr){7-8}

\# & \textbf{\%}
& \textbf{Agr}  & \textbf{Time} 
& \textbf{Agr}  & \textbf{Time} 
& \textbf{Agr}  & \textbf{Time} 

\\ 
\hline %

\multicolumn{2}{c|}{10\%} 
& \textbf{0.71} & 0.18 	 

& 0.37 & 0.0 	 

& \textbf{0.71} & 0.62 	 
 \\
\multicolumn{2}{c|}{30\%} 
& 0.69 & 0.21 	 

& 0.61 & 0.0 	 

& \textbf{0.73} & 0.63 	 
 \\
\multicolumn{2}{c|}{50\%} 
& 0.76 & 0.25 	 

& 0.73 & 0.0 	 

& \textbf{0.78} & 0.63 	 
 \\
\multicolumn{2}{c|}{70\%} 
& 0.82 & 0.33 	 

& 0.85 & 0.0 	 

& \textbf{0.86} & 0.63 	 
 \\
\multicolumn{2}{c|}{100\%} 
& 0.89 & 0.48 	 

& \textbf{0.93} & 0.0 	 

& 0.91 & 0.63 	 
 \\\hline
\multicolumn{2}{c|}{AVG} & 0.77 & 0.29& 0.70 & 0.0& \textbf{0.80} & 0.63 %
\\ \bottomrule
\end{tabular}

         \caption{Non-noisy observations (optimal plans).}
        \label{tab:baselines-optimal-noisefree}
        \end{center}
    \end{subtable}
    \begin{subtable}[t]{.45\textwidth}
        \begin{center}
        \fontsize{8.}{9.}\selectfont
        \setlength\tabcolsep{2pt}
\begin{tabular}{cc|cc|cc|cc}
\toprule
\multicolumn{2}{c}{} %
& \multicolumn{2}{c|}{\rg}
& \multicolumn{2}{c|}{\pom}
& \multicolumn{2}{c}{\holmcsg{}}\\
\cmidrule(lr){3-4} \cmidrule(lr){5-6} \cmidrule(lr){7-8}

\# & \textbf{\%}
& \textbf{Agr}  & \textbf{Time} 
& \textbf{Agr}  & \textbf{Time} 
& \textbf{Agr}  & \textbf{Time} 

\\ 
\hline %

\multicolumn{2}{c|}{10\%} 
& 0.25 & 0.16 	 

& 0.26 & 0.0 	 

& \textbf{0.49} & 0.62 	 
 \\
\multicolumn{2}{c|}{30\%} 
& 0.31 & 0.17 	 

& 0.48 & 0.0 	 

& \textbf{0.55} & 0.63 	 
 \\
\multicolumn{2}{c|}{50\%} 
& 0.35 & 0.17 	 

& 0.63 & 0.0 	 

& \textbf{0.68} & 0.63 	 
 \\
\multicolumn{2}{c|}{70\%} 
& 0.43 & 0.2 	 

& 0.78 & 0.0 	 

& \textbf{0.81} & 0.63 	 
 \\
\multicolumn{2}{c|}{100\%} 
& 0.42 & 0.21 	 

& \textbf{0.89} & 0.0 	 

& 0.88 & 0.63 	 
 \\\hline
\multicolumn{2}{c|}{AVG} & 0.35 & 0.18& 0.61 & 0.0& \textbf{0.68} & 0.63 %
\\ \bottomrule
\end{tabular}

         \caption{Noisy observations (optimal plans).}
        \label{tab:baselines-optimal-noisy}
        \end{center}
    \end{subtable}
    \end{center}
    \caption{Results for baseline approaches (RG and POM) and \holmcsg{} for optimal data sets.}
    \label{tab:baselines-optimal}
\end{table}
\begin{table}[tb]
    \begin{center}
        \fontsize{8.}{9.}\selectfont
        \setlength\tabcolsep{2pt}
\begin{tabular}{cc|cc|cc|cc}
\toprule
\multicolumn{2}{c}{} %
& \multicolumn{2}{c|}{\rg}
& \multicolumn{2}{c|}{\pom}
& \multicolumn{2}{c}{\holmcsg{}}\\
\cmidrule(lr){3-4} \cmidrule(lr){5-6} \cmidrule(lr){7-8}

\# & \textbf{\%}
& \textbf{Agr}  & \textbf{Time} 
& \textbf{Agr}  & \textbf{Time} 
& \textbf{Agr}  & \textbf{Time} 

\\ 
\hline %

\multicolumn{2}{c|}{10\%} 
& 0.60 & 0.19 	 

& 0.42 & 0.0 	 

& \textbf{0.66} & 0.62 	 
 \\
\multicolumn{2}{c|}{30\%} 
& 0.67 & 0.22 	 

& 0.63 & 0.0 	 

& \textbf{0.73} & 0.63 	 
 \\
\multicolumn{2}{c|}{50\%} 
& 0.74 & 0.29 	 

& 0.74 & 0.0 	 

& \textbf{0.76} & 0.63 	 
 \\
\multicolumn{2}{c|}{70\%} 
& 0.81 & 0.42 	 

& \textbf{0.83} & 0.0 	 

& \textbf{0.83} & 0.63 	 
 \\
\multicolumn{2}{c|}{100\%} 
& 0.86 & 0.71 	 

& \textbf{0.90} & 0.0 	 

& 0.88 & 0.64 	 
 \\\hline
\multicolumn{2}{c|}{AVG} & 0.74 & 0.37& 0.70 & 0.0& \textbf{0.77} & 0.63 %
\\ \bottomrule
\end{tabular}

        \caption{Results for baseline approaches (RG and POM) and \holmcsg{} for suboptimal-plan data sets and non-noisy observations.}
        \label{tab:baselines-sub-optimal-noisefree}
    \end{center}
\end{table}

In order to assess the effectiveness of our \textit{Linear Programming} approaches, we have run two goal recognition approaches that serve as baselines for performance. 
Specifically, we use a \textit{Classical Planning} approach~\cite{ramirez2010probabilistic} and recent state-of-the-art approaches that rely on \textit{landmarks}~\cite{pereira2020landmarks}. 
The approach from~\citet{ramirez2010probabilistic} which computes the solution set using a translation of the recognition problem into a \textit{Classical Planning} and a modified search procedure to compute cost differences between an optimal plan and a complying plan. 
We refer to this approach as RG on the tables. 
At a high level, their approach is similar to ours in that they compare the costs of complying and non-complying plans. 
While they compute the actual cost of such plans, we instead approximate the costs using our operator counting framework without actually computing a complete plan. 
By contrast, the recent state of the art of \citet{pereira2020landmarks} is similar to our approach in its avoidance of searching for full plans, and instead, it relies on counting observed landmarks on the observations, weighing them by their information value. 
We refer to this approach as POM in the tables. 
Each of these baselines provides distinct advantages to the goal recognition process. 
While RG provides superior agreement ratios on noise-free as well as in low-observability settings, POM is orders of magnitude faster and more robust to noise. 
The superior agreement ratio from RG, especially in low observability settings is largely due to its computation of complete plans, with its resulting of information quality available for inference. 
By contrast, POM avoids the expensive search procedure by relying on observing delete-relaxed landmarks. 
However, in low-observability settings, if the recognizer does not observe landmarks, the accuracy of the goal inference decreases substantially. 

Tables~\ref{tab:baselines-optimal}--\ref{tab:baselines-sub-optimal-noisefree} show the results for these approaches in the same data sets used in Section~\ref{sec:evaluation}. 
For convenience, we replicate the column showing our improved approach from Tables~\ref{tab:lm1}--\ref{tab:lm2}, so the reader can compare them side by side. 
We can see that, on average, the new approaches dominate both baselines in terms of agreement ratio when averaged across domains. 
While it does not have the same dramatic runtime advantage as the POM approach, it achieves a 2x speedup over the RG approach. 
Nevertheless, its agreement ratio at low observability and noisy settings is substantial in most domains. 
Ultimately, the additional information we include in the LP over simply observing landmarks as POM does prove advantageous. 
Similarly, computing the additional information is computationally more efficient than the full search performed by RG, even accounting for the extra effort in the LP to find noise-coping operator counts. 

\begin{table}[!htb]
    \begin{center}
        \fontsize{8.}{9.}\selectfont
        \setlength\tabcolsep{2pt}
\begin{tabular}{cc|cc|cc|cc}
\toprule
\multicolumn{2}{c}{} %
& \multicolumn{2}{c|}{\rg}
& \multicolumn{2}{c|}{\pom}
& \multicolumn{2}{c}{\holmcsg{}}\\
\cmidrule(lr){3-4} \cmidrule(lr){5-6} \cmidrule(lr){7-8}

\# & \textbf{\%}
& \textbf{Agr}  & \textbf{Time} 
& \textbf{Agr}  & \textbf{Time} 
& \textbf{Agr}  & \textbf{Time} 

\\ 
\hline %

\multirow{5}{*}{\rotatebox[origin=c]{90}{\textsc{blocks}}} 
	 & 10

& \textbf{0.42} & 0.3 	 

& 0.05 & 0.0 	 

& 0.37 & 1.46 	 

	\\ & 30

& \textbf{0.49} & 0.33 	 

& 0.22 & 0.0 	 

& 0.35 & 1.47 	 

	\\ & 50

& \textbf{0.55} & 0.4 	 

& 0.28 & 0.0 	 

& 0.42 & 1.47 	 

	\\ & 70

& \textbf{0.63} & 0.53 	 

& 0.38 & 0.0 	 

& 0.43 & 1.48 	 

	\\ & 100

& \textbf{0.74} & 0.76 	 

& 0.51 & 0.0 	 

& 0.54 & 1.48 	 
 \\ \hline
\multirow{5}{*}{\rotatebox[origin=c]{90}{\textsc{depots}}} 
	 & 10

& 0.02 & 0.05 	 

& 0.17 & 0.0 	 

& \textbf{0.37} & 0.62 	 

	\\ & 30

& 0.07 & 0.06 	 

& 0.21 & 0.0 	 

& \textbf{0.40} & 0.62 	 

	\\ & 50

& 0.01 & 0.05 	 

& 0.51 & 0.0 	 

& \textbf{0.65} & 0.63 	 

	\\ & 70

& 0.00 & 0.05 	 

& 0.54 & 0.0 	 

& \textbf{0.77} & 0.63 	 

	\\ & 100

& 0.01 & 0.05 	 

& 0.83 & 0.0 	 

& \textbf{0.93} & 0.63 	 
 \\ \hline
\multirow{5}{*}{\rotatebox[origin=c]{90}{\textsc{driverlog}}} 
	 & 10

& 0.21 & 0.07 	 

& 0.23 & 0.0 	 

& \textbf{0.40} & 0.47 	 

	\\ & 30

& 0.28 & 0.08 	 

& 0.45 & 0.0 	 

& \textbf{0.55} & 0.47 	 

	\\ & 50

& 0.12 & 0.07 	 

& 0.54 & 0.0 	 

& \textbf{0.60} & 0.48 	 

	\\ & 70

& 0.22 & 0.08 	 

& 0.64 & 0.0 	 

& \textbf{0.69} & 0.48 	 

	\\ & 100

& 0.19 & 0.08 	 

& 0.58 & 0.0 	 

& \textbf{0.68} & 0.48 	 
 \\ \hline
\multirow{5}{*}{\rotatebox[origin=c]{90}{\textsc{dwr}}} 
	 & 10

& 0.23 & 0.12 	 

& 0.33 & 0.0 	 

& \textbf{0.41} & 0.55 	 

	\\ & 30

& 0.09 & 0.18 	 

& \textbf{0.56} & 0.0 	 

& 0.35 & 0.55 	 

	\\ & 50

& 0.11 & 0.23 	 

& \textbf{0.75} & 0.0 	 

& 0.45 & 0.56 	 

	\\ & 70

& 0.10 & 0.14 	 

& \textbf{0.69} & 0.0 	 

& 0.59 & 0.56 	 

	\\ & 100

& 0.03 & 0.17 	 

& \textbf{0.88} & 0.0 	 

& 0.78 & 0.57 	 
 \\ \hline
\multirow{5}{*}{\rotatebox[origin=c]{90}{\textsc{ipc-grid}}} 
	 & 10

& 0.12 & 0.07 	 

& 0.54 & 0.0 	 

& \textbf{0.77} & 0.58 	 

	\\ & 30

& 0.08 & 0.07 	 

& 0.72 & 0.0 	 

& \textbf{0.77} & 0.59 	 

	\\ & 50

& 0.04 & 0.05 	 

& 0.85 & 0.0 	 

& \textbf{0.88} & 0.59 	 

	\\ & 70

& 0.02 & 0.05 	 

& \textbf{0.90} & 0.0 	 

& \textbf{0.90} & 0.59 	 

	\\ & 100

& 0.04 & 0.05 	 

& 0.92 & 0.0 	 

& \textbf{0.94} & 0.59 	 
 \\ \hline
\multirow{5}{*}{\rotatebox[origin=c]{90}{\textsc{ferry}}} 
	 & 10

& 0.31 & 0.06 	 

& 0.26 & 0.0 	 

& \textbf{0.38} & 0.40 	 

	\\ & 30

& 0.47 & 0.08 	 

& 0.44 & 0.0 	 

& \textbf{0.51} & 0.40 	 

	\\ & 50

& 0.66 & 0.14 	 

& 0.69 & 0.0 	 

& \textbf{0.81} & 0.40 	 

	\\ & 70

& 0.70 & 0.26 	 

& 0.78 & 0.0 	 

& \textbf{0.87} & 0.40 	 

	\\ & 100

& 0.71 & 0.60 	 

& 0.82 & 0.0 	 

& \textbf{0.94} & 0.40 	 
 \\ \hline
\multirow{5}{*}{\rotatebox[origin=c]{90}{\textsc{logistics}}} 
	 & 10

& 0.28 & 0.16 	 

& 0.41 & 0.0 	 

& \textbf{0.62} & 0.69 	 

	\\ & 30

& 0.12 & 0.07 	 

& 0.81 & 0.0 	 

& \textbf{0.84} & 0.70 	 

	\\ & 50

& 0.03 & 0.06 	 

& 0.90& 0.0 	 

& \textbf{0.95} & 0.70 	 

	\\ & 70

& 0.00 & 0.06 	 

& \textbf{0.99} & 0.00 	 

& 0.96 & 0.71 	 

	\\ & 100

& 0.00 & 0.06 	 

& \textbf{1.0} & 0.0 	 

& \textbf{1.0} & 0.71 	 
 \\ \hline
\multirow{5}{*}{\rotatebox[origin=c]{90}{\textsc{miconic}}} 
	 & 10

& 0.47 & 0.09 	 

& 0.35 & 0.0 	 

& \textbf{0.54} & 0.43 	 

	\\ & 30

& 0.64 & 0.12 	 

& 0.69 & 0.0 	 

& \textbf{0.80} & 0.43 	 

	\\ & 50

& 0.87 & 0.16 	 

& \textbf{0.93} & 0.0 	 

& \textbf{0.93} & 0.43 	 

	\\ & 70

& \textbf{0.98} & 0.23 	 

& 0.94 & 0.0 	 

& 0.93 & 0.44 	 

	\\ & 100

& \textbf{1.0} & 0.4 	 

& \textbf{1.0} & 0.0 	 

& \textbf{1.0} & 0.44 	 
 \\ \hline
\multirow{5}{*}{\rotatebox[origin=c]{90}{\textsc{rovers}}} 
	 & 10

& 0.37 & 0.07 	 

& 0.44 & 0.0 	 

& \textbf{0.52} & 0.47 	 

	\\ & 30

& 0.40 & 0.08 	 

& 0.51 & 0.0 	 

& \textbf{0.66} & 0.47 	 

	\\ & 50

& 0.49 & 0.08 	 

& 0.72 & 0.0 	 

& \textbf{0.86} & 0.48 	 

	\\ & 70

& 0.26 & 0.07 	 

& 0.89 & 0.0 	 

& \textbf{0.94} & 0.47 	 

	\\ & 100

& 0.34 & 0.09 	 

& 0.90& 0.0 	 

& \textbf{1.0} & 0.48 	 
 \\ \hline
\multirow{5}{*}{\rotatebox[origin=c]{90}{\textsc{satellite}}} 
	 & 10

& 0.41 & 0.06 	 

& 0.29 & 0.0 	 

& \textbf{0.52} & 0.41 	 

	\\ & 30

& 0.54 & 0.07 	 

& 0.51 & 0.0 	 

& \textbf{0.59} & 0.41 	 

	\\ & 50

& 0.61 & 0.08 	 

& 0.66 & 0.0 	 

& \textbf{0.77} & 0.41 	 

	\\ & 70

& 0.63 & 0.09 	 

& 0.78 & 0.0 	 

& \textbf{0.89} & 0.41 	 

	\\ & 100

& 0.47 & 0.09 	 

& 0.92 & 0.0 	 

& \textbf{0.97} & 0.41 	 
 \\ \hline
\multirow{5}{*}{\rotatebox[origin=c]{90}{\textsc{sokoban}}} 
	 & 10

& 0.13 & 0.71 	 

& 0.25 & 0.01 	 

& \textbf{0.30} & 0.91 	 

	\\ & 30

& 0.12 & 0.56 	 

& 0.29 & 0.01 	 

& \textbf{0.40} & 0.93 	 

	\\ & 50

& 0.01 & 0.86 	 

& 0.46 & 0.01 	 

& \textbf{0.50} & 0.95 	 

	\\ & 70

& 0.06 & 1.38 	 

& \textbf{0.58} & 0.01 	 

& 0.54 & 0.98 	 

	\\ & 100

& 0.04 & 0.79 	 

& \textbf{0.77} & 0.01 	 

& 0.73 & 1.0 	 
 \\ \hline
\multirow{5}{*}{\rotatebox[origin=c]{90}{\textsc{zeno}}} 
	 & 10

& 0.45 & 0.13 	 

& 0.31 & 0.0 	 

& \textbf{0.50} & 0.51 	 

	\\ & 30

& 0.61 & 0.15 	 

& 0.57 & 0.0 	 

& \textbf{0.70} & 0.51 	 

	\\ & 50

& 0.75 & 0.18 	 

& 0.73 & 0.0 	 

& \textbf{0.86} & 0.52 	 

	\\ & 70

& 0.82 & 0.23 	 

& 0.89 & 0.0 	 

& \textbf{0.92} & 0.52 	 

	\\ & 100

& 0.85 & 0.32 	 

& 0.90& 0.0 	 

& \textbf{0.94} & 0.52 	 
 \\ \hline
\multicolumn{2}{c|}{AVG} & 0.34 & 0.21& 0.61 & 0.0& \textbf{0.69} & 0.63 %
\\ \bottomrule
\end{tabular}

         \caption{Comparison with baseline approaches for suboptimal plans and noisy observations.}
        \label{tab:baselines-sub-optimal-noisy}
    \end{center}
    \vspace{-9pt}
    \label{tab:baselines-sub-optimal}
\end{table}

\section{Related Work}

Seminal work on \textit{Goal} and \textit{Plan Recognition} rely on the \textit{plan-library} formalism~\cite{AAAI_KautzA86,AIJ_Bayesian_CharniakG93,NewModel_GoldmanGM99}, i.e., a hierarchical formalism (i.e., ``hierarchical recipe'') that defines a collection of plans to achieve a set of goals or tasks. 
Existing library-based approaches employ different types of hierarchical formalisms for recognizing goals and plans, such as context-free grammars~\cite{UAI_PynadathW00,AIJ_GeibG09}, HTN-like formalisms using entropy and query methods~\cite{IJCAI_AvrahamiZilberbrand2005,IJCAI_MirskySGK16,AIJ_MirskySGK18}, purely HTN planning~\cite{HollerHTNRec_2018}, etc. 

\citet{Ecai_Hong00,JAIR_Hong01} veers away from the conventional reliance on \textit{plan-libraries} for recognizing goals and plans. 
They instead formalize recognition problems using a \textit{planning domain theory}, a notably more flexible formalism to formalize the knowledge of how to achieve goals. 
Similarly, \citet{ramirez2009plan,ramirez2010probabilistic} adopt \textit{planning domain theory} for recognizing goals and plans, and developed a robust probabilistic framework that enables the use of off-the-shelf planning techniques to perform the recognition process. 
The work of \citet{ramirez2009plan,ramirez2010probabilistic} was instrumental in shaping the landscape of \textit{Plan Recognition and Planning}, laying the groundwork for subsequent research in this field.
Inspired by \textit{Plan Recognition and Planning}, AUTOGRAPH (AUTOmatic Goal Recognition with A Planning Heuristic) from \citet{STAIRS_PattisonL10} is a probabilistic heuristic-based recognition approach able to recognize goals without exhaustively enumerating goal hypotheses. 

\citet{martin2015fast} developed a planning-based goal recognition approach that propagates cost and interaction information in a planning graph, and uses this information to compute posterior probabilities for a set of possible goals. 
The approach of \citet{martin2015fast} stands out as the pioneering in the literature that obviates calling a planner to perform the recognition task, resulting in a very fast recognition approach.
\citet{sohrabi2016plan} extended the probabilistic framework of \citet{ramirez2010probabilistic}, and developed a novel probabilistic recognition approach that deals explicitly with unreliable and spurious observations (i.e., noisy or missing observations), and is able to recognize both goals and plans. 
Their approach computes multiple high-quality plans (using a top-K planner that generates multiple K plans) to compute posterior probabilities over the goals. 
\citet{vered2016_OnlineMirroring} introduced the concept of \textit{Mirroring} in the context of online goal recognition. 
\citet{pereira2017landmark,pereira2020landmarks} developed recognition approaches that rely on the concept of \textit{landmarks}, and like the approach of \citet{martin2015fast}, such landmark-based approaches abstain from using a full-fledged planning process for the recognition process, yielding in very fast and yet accurate recognition approaches. 
Our approaches draw inspiration from the landmark-based approaches of \citet{pereira2017landmark,pereira2020landmarks}, but notably, the main difference is that our LP-based approaches rely on much ``stronger'' constraints, i.e., \textit{operator-counting} constraints, induced by the given observations and goal hypothesis.
Advances in \textit{Plan Recognition as Planning} expanded to other planning settings and environments with different assumptions. 
These include \textit{path-planning} recognition \cite{AAMAS_MastersS17,masters2019JAIR}, the recognition of plans and goals in \textit{continuous domain models}~\cite{vered2016_OnlineMirroring,Vered_2017gj,VeredPMKM_AAMAS18,KaminkaV_AAAI18}, goal recognition in \textit{incomplete and possibly incorrect domain models}~\cite{AAAI2018_PereiraMeneguzzi,PereiraPM_ICAPS_19}, recognition of goals for \textit{boundedly rational} and \textit{irrational agents}
\cite{masters2019_AAMAS_Rational_Irrational,NIPS_ZhiXuanMSTM20}, recognition of \textit{temporally extended goals}~\cite{Pereira_APIN_2023}, and recognition of goals with \textit{timing information}~\cite{ZhangKL_GR_Timing_23}.

Recent cutting-edge developments in \textit{Machine Learning} have extended their influence into \textit{Goal} and \textit{Plan Recognition}, leading to the emergence of \textit{model-free} recognition approaches.
\citet{IJCNN_AmadoPAMGM18} extends the \textit{latent-space planning} architecture of \citet{asai2018latent} for goals recognition in image-based domain models using off-the-shelf recognition approaches.
\citet{PereiraVMR19} developed recognition approaches in continuous control models with \textit{approximate transition functions}, addressing the influence of inaccuracies (imperfections) in learned control models on the effectiveness of recognizing goals.
\citet{AAMAS_PolyvyanyySLS20} employs \textit{process mining} techniques for ``model-approximate'' recognition, involving the extraction of crucial information from event logs and traces to discover models for recognizing goals. 
This work inspired recent research on \textit{process mining} for goal recognition, such as \citet{AIJ_MiningSuPLSB23}, and the usage of \textit{probabilistic trace alignment} for goal recognition~\cite{ICPM_KoMMPP23}.
The approach from \citet{AAAI_ShvoLIM21} learns interpretable sequence classifiers using \textit{finite state automata} in which they show that such an approach can be used for both recognizing goals and behavior classification.
GRAQL (Goal Recognition as Q-Learning) from \citet{AAAI_RL_AmadoMM22} uses learned Q-values instead of explicit goals from traditional recognition approaches. 
The recognition process involves minimizing the distance between observation sequences and the Q-values of the goal hypotheses' policies.
\citet{AAAI_AmadoPM23} develop an approach that combines \textit{learning statistical prediction} and \textit{symbolic reasoning} for performing the tasks of goal and plan recognition simultaneously.
Finally, \citet{ICAPS_GRNet_ChiariGPPSO23} frame the task of goal recognition as a classification task using the \textit{Long Short-Term Memory} (LSTM) recurrent \emph{Deep Neural Network}. %

\clearpage %

\section{Conclusions}

In this article, we develop a comprehensive framework for goal recognition based on \textit{Linear Programming} constraints for goal recognition. 
Specifically, we build upon the \textit{Operator-Counting} framework from \citet{pommerening2014lp}, which our previous work adapted for goal recognition problems~\cite{Santos2021}. 
Our key contributions are threefold. 
First, we provided a comprehensive theoretical characterization of \textit{Operator-Counting} heuristics for goal recognition, proving key properties which our previous work only intuited. 
Second, we expanded the basic \textit{Operator-Counting} constraints with new sets of constraints for \textit{Landmarks in Goal Recognition}~\cite{pereira2020landmarks} that allow our goal recognition approaches to improve their accuracy. 
Third, we comprehensively study the properties of various sources of constraints in an expanded benchmark. 
This expanded benchmark and additional analyses help us understand the contribution of various types of additional constraints to the goal recognition process. 

This article substantially deepens our understanding of landmark-based constraints, however, there substantial scope for further work. 
Specifically, while our seminal work on operator counting constraints for goal recognition uses various sources of \textit{Operator-Counting} constraints, including \textit{state-equation} constraints~\cite{bonet2013admissible}, \textit{post-hoc} optimization~\cite{florian2013posthoc} and landmarks~\cite{bonet2014flow}, this article focuses on landmark constraints. 
This is motivated by our earlier findings that landmark constraints contribute more towards goal recognition tasks than other sources of constraints~\cite{Santos2021}. 
Thus, our future work consists of a deeper investigation on various refinements of such constraint types for goal recognition heuristics.

\acks{
Felipe Meneguzzi acknowledges support from CNPq with projects 407058/2018-4 (Universal) and 302773/2019-3 (PQ Fellowship). 
Andr\'e G. Pereira acknowledges support from FAPERGS with project 17/2551-0000867-7. 
This study was financed in part by the Coordena\c c\~ao de Aperfei\c coamento de Pessoal de N\'ivel Superior --- Brasil (CAPES) --- Finance Code 001. 
}

\vskip 0.2in

\bibliographystyle{theapa}

\appendix
\clearpage
\section{Formalism Summary}

\begin{itemize}
    \item[$\variables$] discrete finite-domain variables;
    \item[$\operators$] {\sasplus} \emph{operators};
    \item[$\initialstate$] initial state;
    \item[$\goalstate$] goal state (this is actually a conjunctive formula, not a state, for the states that qualify as a goal see $\goalstates$ below);
    \item[$\goalstates$] subset of states consistent with $\goalstate$;
    \item[$\cost$] cost function;
    \item[$\planningtask$] planning task $\planningtask = \tuple{\variables, \operators, \initialstate, \goalstate, \cost}$;
    \item[$\tuple{\variable,v}$] atom consisting of a variable $\variable \in \variables$ and one of its values $v \in \dom{\variable}$;
    \item[$\facts$] The set if all atoms in a domain;
    \item[$s$] A state;
    \item[$\states$] the set of all (complete) states over \variables;
    \item[$o = \tuple{p, e}$] operator: $p = \pre(o)$ is the set of preconditions, and $e = \post(o)$ is the set of effects;
    \item[$s'= s\exec{o}$] state resulting from executing an operator $o$;
    \item[$\plan$] an s-plan that starts at state $s$ and results in a state $\goalstate$;
    \item[$\transgraph_{\planningtask}$] \emph{transition system} $\transgraph_{\planningtask} = \tuple{\vertices, \edges, \initialstate, \goalstates}$ induced by $\planningtask$;
    \item[$\edges$] set of transitions; 
    \item[$\h$] A heuristic function $\h : S \rightarrow \Real\cup\{\infty\}$ from states $S$ into a real (and possibly infinite) value; 
    \item[$\hoptimal(s)$] The perfect/optimal heuristic computed from state $s$;
    \item[$\textup{IP/LP}$] Integer/Linear Program, these programs lead to a number of operator-counting heuristics, depending on the constraints:
    \begin{itemize}
        \item[$\hseq$] An operator-counting heuristic whose LP includes state equation constraints;
        \item[$\hlmc$] An operator-counting heuristic whose LP includes landmark constraints; 
        \item[$\hpho$] An operator-counting heuristic whose LP includes post-hoc optimization constraints; 
        \item[$\hflow$] An operator-counting heuristic whose LP includes network flow constraints;
    \end{itemize} 
    \item[$\variables$] A set of real-valued and integer variables, of which:
    \begin{itemize}
        \item[$\Y{o}$] A non-negative \emph{operator-counting} variable, there is one such variable for each $o \in \operators$;
        \item[$\Yobs{o}$] A non-negative \emph{observation-counting} variable, there is one such variable for each $o \in \operators$;
        \item[$\varU{o}$] An operator counting variable indicating whether $o \in \operators$ is part of a delete-relaxed plan $\plan$\frm{Make sure that the plans referred to here are always delete relaxed};
        \item[$\varR{a}$] A variable indicating that fact $a \in \facts$ is reached by a delete-relaxed plan $\plan$, note that this variable does not count operators per se;
        \item[$\varA{o,a}$] A variable indicating whether $o \in \operators$ is the first operator in a delete-relaxed $\plan$ to achieve fact $a \in \facts$;
        \item[$\varT{o}$] an integer variable denoting the time in which $o \in \operators$ occurs in a delete-relaxed plan $\plan$ for the  first time;
        \item[$\varT{a}$] an integer variable denoting the time in which $a \in \facts$ is achieved for the first time;
    \end{itemize} 
    \item[$\occur_{\plan}(o)$] The fact that an operator $o \in \operators$ occurs in a plan $\plan$;
    \item[$\constraints$] A set of operator-counting constraints;
    \item[$\landmark$] A disjunctive action landmark;
    \item[$\grplanningtask$] A planning task without a goal condition (used to represent the domain for a goal recognition task)
    \item[$\grtask$] A goal recognition task $\tuple{\grplanningtask, \goalconditions, \observations}$ with a domain $\grplanningtask$, goal hypotheses $\goalconditions$, and observations $\observations$
    \item[$\goalconditions$] A set of goal conditions $\goalstate$ 
    \item[$\observations$] A sequence of observations
    \item[$\rgoal$] The \emph{reference} goal for a goal recognition task $\grtask$
    \item[$\unreliability$] The level of unreliability or noise assumed in the observations;
\end{itemize} %

\end{document}